\pdfoutput=1
\documentclass[12pt]{article}

\usepackage[preprint]{neurips_2024}

\usepackage[utf8]{inputenc} 
\usepackage[T1]{fontenc}    
\usepackage{hyperref}       
\usepackage{url}            
\usepackage{booktabs}       
\usepackage{amsfonts}       
\usepackage{nicefrac}       
\usepackage{microtype}      
\usepackage{xcolor}         
\usepackage{wrapfig}
\usepackage{subcaption}
\usepackage{enumitem}
\usepackage[normalem]{ulem}

\usepackage{times}
\usepackage{subcaption}
\usepackage{caption}

\usepackage{amsthm}
\usepackage[utf8]{inputenc} 
\usepackage[T1]{fontenc}    
\usepackage{hyperref}       
\usepackage{url}            
\usepackage{booktabs}       
\usepackage{amsfonts,amsmath,amssymb}       
\usepackage{nicefrac}       
\usepackage{microtype}      
\usepackage{xcolor}         
\usepackage{txfonts}
\usepackage{graphicx}
\usepackage{wrapfig,bm,comment,color}
\usepackage{breakurl,epsfig,epsf,fmtcount,semtrans,multirow,boldline}
\usepackage{tcolorbox}
\tcbuselibrary{skins}
\usepackage{tikz}

\definecolor{darkred}{RGB}{150,0,0}
\definecolor{darkgreen}{RGB}{0,150,0}
\definecolor{darkblue}{RGB}{0,0,200}
\hypersetup{colorlinks=true, linkcolor=darkred, citecolor=darkgreen, urlcolor=darkblue}

\newtheorem{theorem}{Theorem}

\newtheorem{assumption}{Assumption}

\newtheorem{lemma}{Lemma}
\newtheorem{corollary}{Corollary}
\newtheorem{proposition}{Proposition}
\newtheorem{definition}{Definition}

\newcommand{\beq}{\begin{equation}}
\newcommand{\ba}{\begin{align}}
\newcommand{\ea}{\end{align}}

\newcommand{\barL}{\bar{L}}

\newcommand{\eeq}{\end{equation}}
  
\newcommand{\vct}[1]{\bm{#1}}
\newcommand{\mtx}[1]{\bm{#1}}
\newcommand{\red}{\textcolor{darkred}}
\newcommand{\green}{\textcolor{darkgreen}}
\newcommand{\purple}{\textcolor{purple}}
\newcommand{\blue}{\textcolor{darkblue}}

\newcommand{\gray}[1]{\textcolor{gray}{#1}}
\newcommand{\shaw}[1]{{#1}}

\newcommand{\eps}{\varepsilon}

\newcommand{\st}{\star}
\newcommand{\h}{h}

\newcommand{\A}{{\mtx{A}}}

\newcommand{\Ib}{{{\mtx{I}}}}

\newcommand{\Nc}{{\cal{N}}}

\newcommand{\F}{{\mtx{F}}}

\newcommand{\Db}{{\mtx{D}}}
\newcommand{\dly}[1]{{\mtx{D}}_{#1}}

\newcommand{\Iden}{{\mtx{I}}}
\newcommand{\M}{{\mtx{M}}}
\newcommand{\nb}{{\vct{n}}}
\newcommand{\z}{{\vct{z}}}

\newcommand{\bt}{{\boldsymbol{\beta}}}

\newcommand{\Bc}{\mathcal{B}}
\newcommand{\Sc}{\mathcal{S}}

\newcommand{\Nn}{\mathcal{N}}
\newcommand{\vb}{\vct{v}}

\newcommand{\Fb}{\vct{F}}

\newcommand{\Ic}{{\mathcal{I}}}

\newcommand{\li}{\left<}
\newcommand{\ri}{\right>}
\newcommand{\s}{\vct{s}}

\newcommand{\ab}{\vct{a}}
\newcommand{\bb}{\vct{b}}
\newcommand{\ub}{{\vct{u}}}

\newcommand{\g}{{\vct{g}}}

\newcommand{\Z}{\mtx{Z}}
\newcommand{\Ks}{{\Kb^{\text{ss}}}}

\newcommand{\Kbl}{\vct{K}^{\text{loc}}}
\newcommand{\m}{\vct{m}}
\newcommand{\x}{\vct{x}}

\newcommand{\y}{\vct{y}}

\newcommand{\W}{\mtx{W}}

\newcommand{\Vc}{{\cal{V}}}

\newcommand{\X}{{\mtx{X}}}

\newcommand{\Vb}{{\mtx{V}}}

\newcommand{\Kb}{{\mtx{K}}}

\newcommand{\qb}{{\vct{q}}}

\newcommand{\norm}[1]{\texttt{norm}(#1)}

\newcommand{\R}{\mathbb{R}}
\newcommand{\Pro}{\mathbb{P}}

\newcommand{\E}{\operatorname{\mathbb{E}}}

\newcommand{\sft}[1]{\mathbb{S}(#1)}

\newcommand{\tn}[1]{\|{#1}\|_{\ell_2}}

\newcommand{\tone}[1]{\|{#1}\|_{\ell_1}}
\newcommand{\tin}[1]{\|{#1}\|_{\ell_\infty}}

\newcommand{\order}[1]{{\cal{O}}(#1)}

\newcommand{\Xb}{\mtx{\bar{X}}}
\newcommand{\xb}{\vct{\bar{x}}}

\usetikzlibrary{matrix}

\usepackage{siunitx}[=v2]
\usepackage{makecell} 

\title{Convolution Augments Attention:\\Solving Associative Recall with One Layer}
\title{Attention with Convolution is All You Need}
\title{All Attention Needs is Convolution}
\title{Convolution is All Attention Needs for\\Language Modeling}
\title{On the Power of Convolution Augmented Transformer}

 \author{Mingchen Li\quad\quad~~~Xuechen Zhang\\University of Michigan\\\texttt{\{milii,zxuechen\}@umich.edu} \And Yixiao Huang\\{UC Berkeley}\\\texttt{{yixiaoh@berkeley.edu}}\And
 \And Samet Oymak\\University of Michigan\\\texttt{oymak@umich.edu}}

\begin{document}
\maketitle
\begin{abstract} The transformer architecture has catalyzed revolutionary advances in language modeling. However, recent architectural recipes, such as state-space models, have bridged the performance gap. Motivated by this, we examine the benefits of Convolution-Augmented Transformer (CAT) for recall, copying, and length generalization tasks. CAT incorporates convolutional filters in the K/Q/V embeddings of an attention layer. Through CAT, we show that the locality of the convolution synergizes with the global view of the attention. Unlike comparable architectures, such as Mamba or transformer, CAT can provably solve the associative recall (AR) and copying tasks using a single layer while also enjoying guaranteed length generalization. We also establish computational tradeoffs between convolution and attention by characterizing how convolution can mitigate the need for full attention by summarizing the context window and creating salient summary tokens to attend. Evaluations on real datasets corroborate our findings and demonstrate that CAT and its variations indeed enhance the language modeling performance.
\end{abstract}










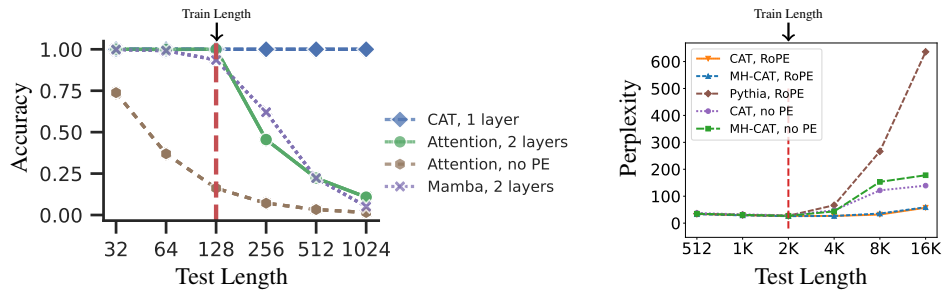
\begin{figure}[h]
    \begin{subfigure}{0.45\textwidth}
    \centering
    \vspace{-20pt}
    \hspace*{-20pt}
        \begin{tikzpicture}
            \node at (-1, 0.15) [scale=0.82] 
            {\includegraphics{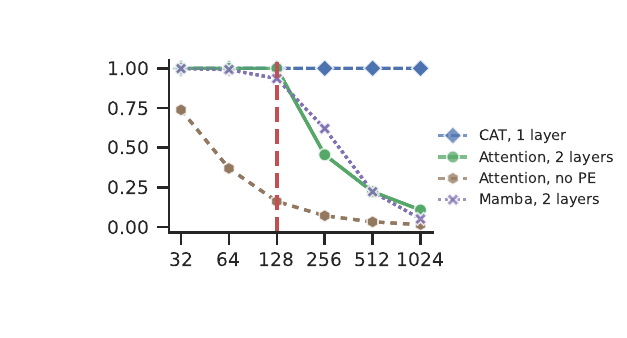}};
            \node at (-1.5, 2) [scale=0.5] {Train Length};
            \draw[->, line width=0.3mm] (-1.5, 1.9) -- (-1.5, 1.65);
            \node at (-1.3, -1.5) [scale=0.9] {Test Length};
            \node at (-4.1, 0.5) [scale=0.9, rotate=90] {Accuracy};

            \node at (6.2, 0.2) [scale=0.25] 
             {\includegraphics{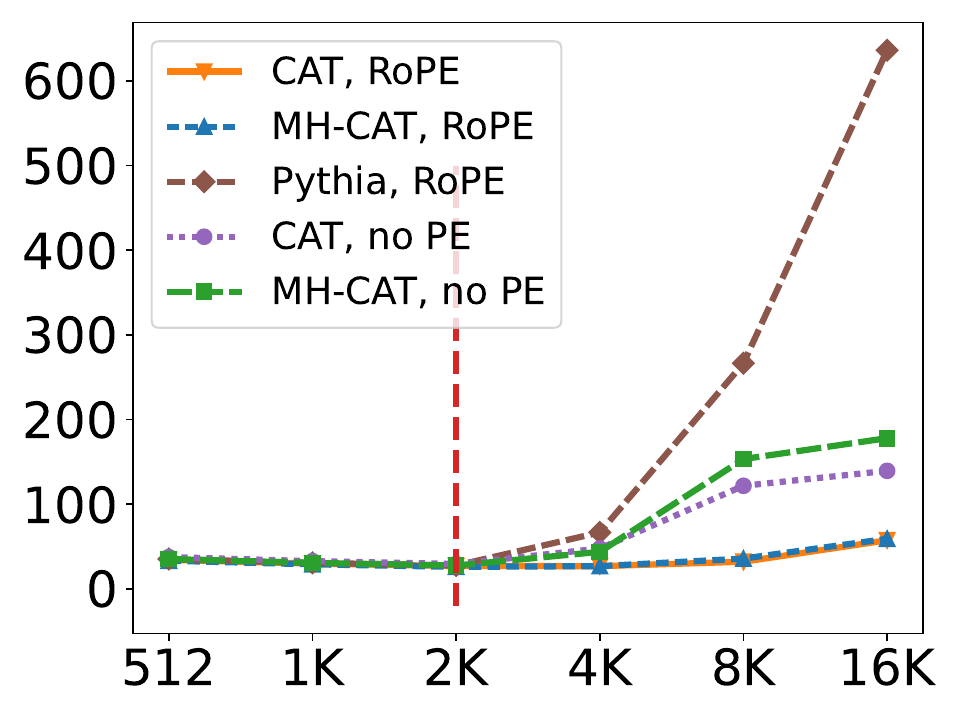}};
            
            \node at (6.1, 2) [scale=0.5] {Train Length};
            \draw[->, line width=0.3mm] (6.1, 1.9) -- (6.1, 1.65);
            \node at (6.4, -1.5) [scale=0.9] {Test Length};
            \node at (4, 0.5) [scale=0.9, rotate=90] {Perplexity};
        \end{tikzpicture}
    \end{subfigure}
    \vspace{-15pt}
    \caption{\small{Evaluations on synthetic and real data. The models are trained on 128 and 2,048 context length (vertical dashed lines) and tested on varying context lengths respectively. \textbf{Left figure:} We conduct synthetic experiments on the Associative Recall task and contrast 1-layer CAT with 2-layers of alternative architectures. The embedding dimension is 128. We find that CAT is the only model that solves AR with length generalization in line with our theory (also see Fig.~\ref{fig:lengen}). \textbf{Right figure:} Evaluations on language modeling where we train CAT models by equipping Pythia with short convolutions (window size 21). Convolution allows the model to pretrain without positional encoding and further improves perplexity when combined with RoPE. Importantly, it also generalizes to longer context lengths more robustly with or without RoPE. For length generalization, we used YaRN~\cite{peng2023yarn} which incorporates position interpolation~\cite{chen2023extending} (for RoPE only) and temperature scaling (see Sec.~\ref{sec:nlp_exp}).}}
    \label{fig:lengen_intro}
    \vspace{-10pt}
\end{figure}
\section{Introduction}\vspace{-2pt}


The attention mechanism is the central component of the transformer architecture \cite{vaswani2017attention} which empowers modern large language models. Through the self-attention layer, all pairs of tokens get to interact with each other which equips the model with a global view of the context window. On the other hand, without positional-encoding (PE), self-attention lacks \emph{locality}. For instance, without PE or causal masking, self-attention layer is permutation-equivariant and does not distinguish between nearby vs distant tokens. In contrast, convolution operator is a well-established primitive that facilitates local feature aggregation based on relative positions and provides a natural alternative to PE. While convolution has enjoyed major success in vision during the last three decades, its explicit use in language modeling is relatively recent \cite{dauphin2017language}. On the other hand, there is a growing recent interest in using convolution-based blocks in language models: For instance, state-space models (SSM) \cite{gu2021efficiently} and linear RNNs \cite{orvieto2023resurrecting} are efficient parameterizations of long convolutional filters. These models have enjoyed significant success in long-range sequence modeling as they provide fast inference and parallelizable training. On the other hand, purely convolutional architectures are known to suffer from recall capability as they lack the global view of the context window \cite{arora2024simple}. These insights motivated a recent push toward hybrid architectures \cite{de2024griffin,arora2024simple,park2024can,zoology2023} that combine the strengths of both attention and convolution-like approaches, including short convolutional filters, SSMs, linear RNNs, or Mamba.


\begin{figure}[t]
    \centering
    \begin{tikzpicture}
        \node at (-3.65,0) [scale=0.265] {\includegraphics[trim=0 250 20 0, clip]{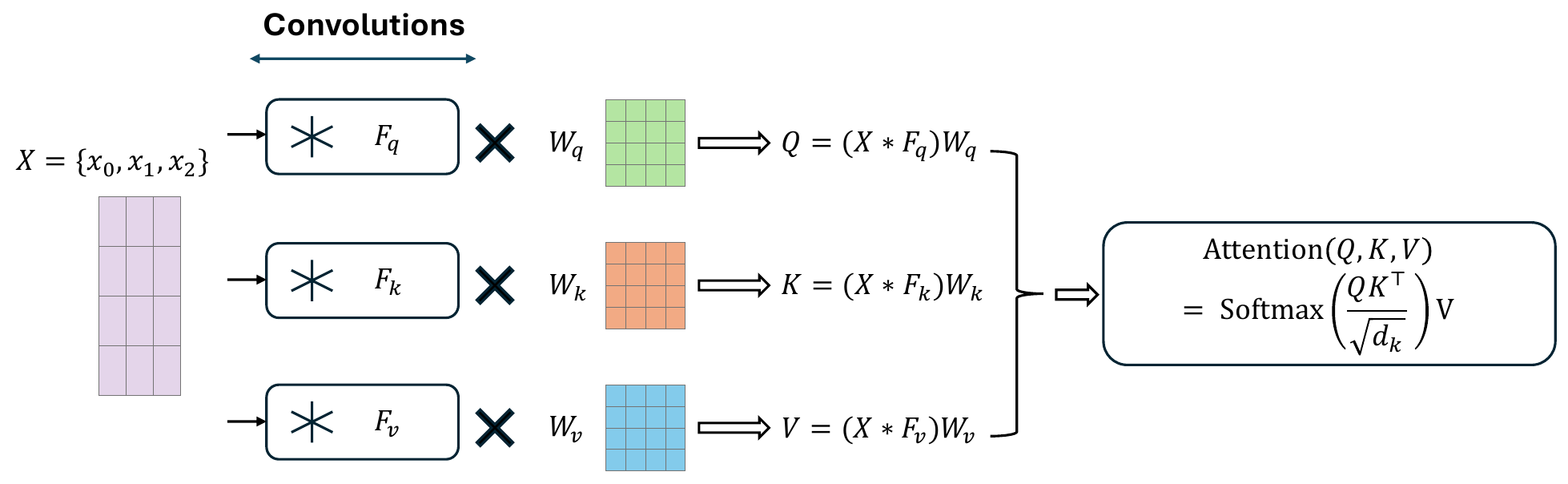}};
        \node at (2.85, -0.15) [scale=0.5] {\includegraphics[trim = 10 0 410 0, clip]{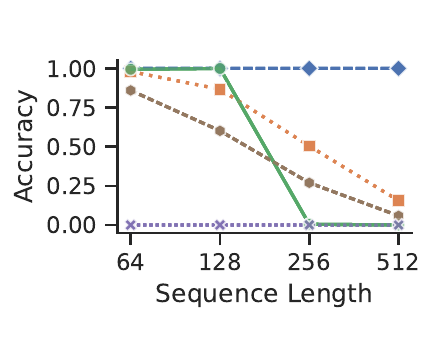}};
        \node at (5.15, 0.15) [scale=0.55] {\includegraphics[trim = 531 57 5 57, clip]{figs/conv_1gram.pdf}};
        \node at (3.2, 1.05) {\scriptsize{Associative Recall Comparison}};
    \end{tikzpicture}\caption{\small{\textbf{Left figure}: Illustration of the Convolution-Augmentated Attention (CAT) block, where separate filters are applied to the K/Q/V embeddings, before self-attention (see Sec.~\ref{sec:methodology} for details). \textbf{Right figure}: Performance of 1-layer CAT models trained on multi-query AR (MQAR, see Sec.~\ref{sec:MQAR} for details) tasks with model embedding dimension 64 and varying sequence length. The LinCAT replaces the standard attention in CAT with linear attention. We observe that the CAT model outperforms the baseline models across all sequence lengths with only 1 layer compared to 2 layers baselines.}
    }\label{fig:intro}\vspace{-10pt}
\end{figure}




In this work, we explore the synergy between attention and convolution which reveals new theoretical principles that inform hybrid architecture design. Specifically, we introduce an intuitive hybrid architecture called \emph{Convolution-Augmented Transformer} (CAT)\footnote{The transformer architecture consists of attention and MLP layers. For theoretical analysis and synthetic experiments, we will entirely focus on the \emph{Convolution Augmented Attention} layer described in Fig.~\ref{fig:intro}. For this reason, we will use the CAT acronym to refer to both Convolution-Augmented Transformer and Attention.}. CAT incorporates convolutional filters to the K/Q/V embeddings of the attention layer as depicted on the left hand side of Figure \ref{fig:intro}. We explore the capabilities of the CAT layer through mechanistic tasks including associative recall (AR), selective copying \cite{gu2023mamba,jing2019gated}, and length generalization. For instance, AR is a fundamental task motivated from the associative memory in cognitive science \cite{ba2016using}. This task underpins critical applications such as bigram retrieval, where a specific sequence, such as \blue{`Rings'} following \purple{`The Lord of the'}, must be correctly retrieved. It is also a generalization of the induction head task \cite{olsson2022context} and known to be crucial for LLM functionality and mechanistic understanding \cite{olsson2022context,fu2022hungry,arora2024simple,nichani2024transformers,poli2024mechanistic}.



We theoretically and empirically show that, within the CAT layer, attention and convolution exhibit strong synergy and complementarity to solve these mechanistic tasks while enjoying length generalization benefits. As a concrete example, the left side of Figure \ref{fig:lengen_intro} displays the AR performance for various test-time sequence lengths. As the sequence length grows, we observe two distinct failure modes: Mamba's accuracy degrades due to its finite state dimension whereas attention-only models degrade due to the length extension bottlenecks of PE. In contrast, CAT maintains perfect accuracy and length generalization because attention and convolution patch these failure modes in a complementary fashion. Overall, we make the following contributions:\vspace{-1pt}
\begin{itemize}[leftmargin=*]
    \item We propose the convolution-augmented attention layer and prove that it can solve the N-gram AR (NAR) and Selective Copying tasks using a single layer (Theorems \ref{thm NAR main} and \ref{sel copy thm}). Comparison to alternatives (Mamba, Based, attention, linear attention) reveals that CAT can uniquely solve NAR with length generalization.\vspace{-1pt}


\item To explain this, we establish a length generalization result on the loss landscape (Theorem \ref{prop len gen new}): Under mild assumptions, all CAT models that solve AR for a particular context length provably generalize to all other context lengths.\vspace{-1pt} 


\item We evaluate CAT on real data and demonstrate that even 1-dimensional short convolutions noticeably aids language modeling: In line with theory, convolution enables the model to train stably without PE and improves length generalization. We also develop a multihead version of CAT which yields further accuracy improvements (see Table \ref{tab:real_nlp}).\vspace{-1pt}
\item We show that long convolutions, such as SSMs, bring the benefit of \emph{context summarization} and mitigates the need for dense attention: We describe the Landmark CAT model (following Landmark Attention \cite{mohtashami2023landmark}) which first creates landmark/summary tokens through convolution and then attends on these landmarks to efficiently locate the most relevant subsets of the input sequence (Sec.~\ref{sec tradeoff}). Focusing on the AR problem, we characterize fundamental tradeoffs between the embedding dimension, amount of summarization, and the sparsity of attention. Through these, we show that the use of long convolutions can provably enable the success of sparse/cheaper attention.
\end{itemize}

\vspace{-6pt}
\section{Related Works}\label{sec:related}
\vspace{-5pt}


\noindent\textbf{Convolution-like sequence models.} Gated-convolutions \cite{dauphin2017language} and state-space models, such as S4~\cite{gu2021efficiently}, utilize long convolutions to reduce the computational demands associated with attention mechanisms. Performance enhancements have also been achieved through novel filter parametrization techniques~\cite{gupta2022diagonal, gu2022parameterization}. Despite these innovations, challenges in Multi-query Associative Recall (MQAR) prompted the development of input-dependent convolution techniques. Notable developments in this area include, Liquid S4~\cite{hasani2022liquid}, Mamba~\cite{gu2023mamba,dao2024transformers} and \cite{yang2019condconv,kosma2023time} where convolution filters are directly parametrized by inputs and include correlation terms between input tokens to enhance state mixing. \cite{li2022makes} empirically explores the reason underlying the success of convolutional models. \vspace{-1pt}

\noindent\textbf{Expressivity, recall, length generalization.} Recent works \cite{arora2024simple,jelassi2024repeat,zoology2023,fu2022hungry} explore the limitations of purely convolutional models, including Mamba, and demonstrate that, these models inherently lack the capability to solve recall problems unless they have large state dimensions (i.e.~memory). \cite{jelassi2024repeat} also provides a construction for 2-layer self-attention to solve AR with length generalization. Interestingly, this construction uses Hard Alibi, which is a variation of Alibi PE \cite{press2021train} that utilize explicit linear biases in attention. Their Hard Alibi restricts the attention layer to focus on and aggregate only the recent $N$ tokens. In this regard, this construction is related to our short convolution. On the other hand, while this work is constructive, we also prove that CAT has good loss landscape and all CAT solutions to AR provably length generalize. It has also been observed that PE can hurt length generalization and reasoning. In fact, \cite{kazemnejad2024impact} has found NoPE to be viable. On the other hand, in our real data evaluations, we have found pure NoPE to be highly brittle as it either fails to converge or optimization is unreasonably slow. Our AR experiments also corroborate that NoPE by itself is indeed not a viable strategy. \vspace{-1pt}



\noindent{\textbf{Hybrid architectures.}} There is a growing interest in integrating different language modeling primitives to obtain best-of-all-world designs. To this end, mechanistic tasks such as AR, copying, induction head, and in-context learning have been important to demystify the functionalities of language models \cite{olsson2022context,park2024can} and have been utilized to guide architecture design \cite{zoology2023,poli2024mechanistic}. Gating mechanisms have been integrated within convolutional frameworks to enhance the model's selectivity. Models employing gating functions, have shown substantial improvements in AR tasks~\cite{fu2022hungry, poli2023hyena}. Additionally, recent innovations on hybrid architecture, such as BaseConv~\cite{zoology2023, arora2024simple}, GLA~\cite{yang2023gated}, MambaFormer \cite{park2024can}, and \cite{ma2024megalodon,ma2022mega,ren2024samba} have provided more effective solutions to AR tasks. This comprehensive foundation of hybrid architectures informs our exploration into the convolution-attention synergy.

\vspace{-4pt}
\section{Problem Setup}\vspace{-4pt}

\subsection{Convolutional-Augmented Attention}\label{sec:methodology}\vspace{-2pt}

Let us first introduce helpful notation. $\Iden_d$ is the identity matrix of size $d$. $\dly{i}$ denotes the causal delay filter that shifts a signal $\x$ $i$-timesteps forward i.e.~$(\x\ast \dly{i})_j=\x_{j-i}$. For an integer $n \geq 1$, we denote the set $\{0, \dots, n-1 \}$ by $[n]$. We use lower-case and upper-case bold letters (e.g., $\m, \M$) to represent vectors and matrices, respectively. $m_i$ denotes the $i$-th entry of a vector $\m$. 

Below, we introduce the Convolution-Augmented Attention layer, which incorporates learnable filters into the K/Q/V embeddings. Let $\X=[\x_0~\dots~\x_{L-1}]^\top \in\R^{L\times d}$ denote the input to the layer containing $L$ tokens with embedding dimension $d$. Let $\Fb\in\R^W$ denote the convolutional filter with temporal length $W$. We examine two convolution types which handle multi-head attention in different ways:


\noindent $\bullet$ \textbf{1D per-head convolution:} For each attention head, we have a distinct 1D filter $\Fb\in\R^W$. $\Fb$ is applied temporally to each of the $d$ embedding dimensions. This results in $\Fb\ast \X$ where $(\Fb \ast \X)_i = \sum_{j \in [W]} F_j \x_{i - j}$, with $F_j$ being the $j$-th entry of $\Fb$.

\noindent $\bullet$ \textbf{Multi-head convolution:} Suppose we have $H$ sequences $\bar{\X}=[\X_1,\dots,\X_H]\in\R^{L\times d\times H}$ each corresponding to one of the $H$ attention heads. We use a filter $\bar{\Fb}=[\Fb_1,\dots,\Fb_H]\in\R^{W\times H\times H}$. Each $\Fb_i$ is convolved with $\bar{\X}$ to obtain the $i$-th head's output of size $L\times d$.

Observe that both convolution types are identical when there is a single attention head. However, multi-head convolution is more expressive because it mixes the attention heads. In Section \ref{sec:experiments}, we will also examine a variation of multi-head convolution where we mix the attention maps rather than embeddings. The architecture of CAT is illustrated in Fig.~\ref{fig:intro} and is formally defined as follows:

\begin{definition}[Convolution-Augmented Attention (CAT)]\label{cat layer} A CAT layer incorporates learnable convolutional filters to the key/query/value embeddings. For a single-head CAT, the key embeddings are given by $\Kb=(\X\ast\Fb_k)\W_k$ with weights $\Fb_k,\W_k$ (same for query and value embeddings). 
\end{definition}

\begin{table}[b!]\vspace{-8pt}\caption{\small{Illustrative examples of synthetic tasks. In all AR-based tasks, keys and queries are highlighted in \red{red} and the values in \green{green}. For NAR tasks, parentheses denote N-gram queries; note that the parentheses are not part of the input. In SC tasks, signal tokens are in \green{green} and noise tokens in \gray{gray}, and the model begins output when $\bot$ appears in the sequence.}
}\label{tab:syn_tasks}
\begin{center}
\begin{tabular}{c|c|c|c|c}
    \toprule
    & & Input & Query & Output\\
    \midrule
    \multirow{3}{*}{Single Query} & AR & \red{a} \green{2} \red{c} \green{1}& \red{a} & \green{2} \\
    & NAR  &  \red{(a b)}  \green{2}  \red{(b a)}  \green{q} \red{(a a)} \green{4} &  \red{b a}  & \green{q} \\
    & SC & \green{a} \gray{[n] [n]} \green{c} \gray{[n]} \green{k} & $\bot$  & a c k\\
    \midrule
    \multirow{3}{*}{Multi Query} & AR & \red{a} \green{2} \red{c} \green{1}& \red{c a} & \green{1 2} \\
    & NAR  &  \red{(a b)}  \green{2}  \red{(b a)}  \green{q} \red{(a a)} \green{4} &  \red{(b a) (a a)}  & \green{q 4} \\
    \bottomrule
\end{tabular}

\end{center}\vspace{-10pt}
\end{table}

\subsection{Mechanistic Tasks for Language Modeling}
To proceed, we describe the Associative Recall and Selective Copying tasks that will help us mechanistically study CAT. Table \ref{tab:syn_tasks} provides an illustration of these tasks which are adapted from the sequence modeling literature \cite{gu2023mamba,zoology2023,poli2024mechanistic,olsson2022context}. 
\begin{definition}[Associative Recall Problem] \label{AR problem def}
Consider a discrete input sequence $X = [x_0, x_1, \dots, x_{L-1}]$, with tokens drawn from a vocabulary $\Vc$ of size $|\Vc|$. The AR problem is defined as follows: Suppose that there is a unique index $i$ ($0 \leq i < L-1$) such that $x_i = x_{L-1}$. A model $f$ successfully solves the AR problem if $f(X) = x_{i+1}$ for all inputs $X$. In this problem, $x_i$ becomes the key, $x_{i+1}$ is the associated value, and the last token $x_{L-1}$ is the query.\label{ARprob}
\end{definition}
    
Building on the AR problem, we introduce its N-gram variation: The model needs to identify the copy of the last $N$ tokens in the context window and return the associated value.


\begin{definition}[N-gram AR Problem] 
Consider a discrete input sequence $X = [x_0, x_1, \dots, x_{L-1}]$, with tokens drawn from a vocabulary $\Vc$ of size $\Vc|$. Let $X_{\{i, j\}} = [x_i, x_{i+1}, \dots, x_j]$ denote the subsequence of $X$ from index $i$ to $j$. The N-gram associative recall (NAR) problem is formulated as follows: for $X_{\{L-N,L-1\}}$ (which are the last N tokens), there exists a unique index $i$ ($0 \leq i < L-N$) such that $X_{\{i, i+N-1\}} = X_{\{L-N,L-1\}}$. A model $f$ solves NAR if $f(X) = x_{i+N}$ for all inputs $X$.\label{NAR prob def}
\end{definition}

Selective copying (SC) task is originally introduced by \cite{jing2019gated} and it is utilized by the recent Mamba \cite{gu2023mamba} and Griffin \cite{de2024griffin} papers to assess their model's approximation capabilities. In SC, given an input sequence $X$ containing noisy tokens, the model should denoise $X$ and return the \emph{signal tokens} within.
\begin{definition}[Selective Copying]\label{sel copy def} Consider a vocabulary $\Vc$ composed of a set of signal tokens $\Sc$, a set of noise tokens $\Nc$, and special token $\bot$ i.e.~$\Vc=\Sc\cup\Nc\cup\{\bot\}$. Let $X$ be a sequence whose tokens are drawn from $\Sc\cup\Nc$ and let $X_{\Sc}$ be the sub-sequence of $X$ that includes all signal tokens in order. $f$ solves \emph{selective copying} over $\Sc$ if it autoregressively outputs $X_{\Sc}$ following the prompt $[X~\bot]$ for all inputs $X$. {$f$ solves \emph{unique selective copying} if it outputs all unique tokens of $X_{\Sc}$ in order for all $X$.}
\end{definition}

\subsubsection{Multi-Query Associative Recall}\label{sec:MQAR}
In this section, we introduce the multi-query versions of the AR and NAR tasks, abbreviated as MQAR and MQNAR, respectively. In the multi-query (MQ) setting, a model receives multiple queries in a single input and must generate corresponding outputs in a single forward pass, at varying positions in the sequence. This approach was first introduced in \cite{zoology2023}, which demonstrated that while the Mamba model successfully addresses single-query AR tasks, it struggles with MQAR when operating with a limited model dimension. This highlights the increased complexity of MQAR tasks where models need to memorize more sequence information and recall queries at different positions. 



\begin{definition}[Multi-Query Associative Recall (MQAR)]
    Consider a discrete input sequence $X = [x_0, x_1, \dots, x_{L-1}]$ with tokens drawn from a vocabulary $\Vc$. Let $X_{\{i,j\}} = [x_i, \dots, x_j]$ denote a subsequence of $X$ from index $i$ to $j$. The \textbf{multi-query N-gram associative recall (MQNAR)} problem is defined as follows: for every N-gram query $Q_k = X_{{k-N+1} \dots {k}}$, $N\leq k < L$, determine if there exists a $N\leq j < k$ such that $X_{\{j-N+1, j\}} = Q_k$. If so, output the value $x_{j+1}$ as the result, else output a special token to indicate no match is found. A model $f$ solves MQNAR if it outputs the correct values for all $N$-gram queries and all inputs $X$. 
    The standard \textbf{MQAR} problem \cite{zoology2023} is a special instance of MQNAR by setting $N=1$.
\end{definition}

Table~\ref{tab:syn_tasks} provides examples of the synthetic tasks we consider in this work. Specifically, we conduct AR and NAR experiment on their multi-queiry variants to evaluate the model's ability to recall multiple queries. For the selective copying task, we generate the output auto-regressively by predicting the signal tokens in the input sequence after the special token $\bot$.

\section{Provable Benefits of Convolution-Augmented Attention}

Before diving into the theoretical results, we make a few clarifying remarks. We assume that all token embeddings have unit $\ell_2$ norm. Secondly, a CAT layer maps each query to a vector-valued output $f(\X)\in\R^d$. To sample the discrete output token, we will simply return the nearest neighbor in the vocabulary of token embeddings. For associative recall problems, we will use a single head attention layer with weights $\W_q,\W_k$ are chosen as suitably scaled identity matrices. With this choice, attention essentially implements a \emph{nearest neighbor retrieval}. It suffices for the theory thanks to the simple nature of the AR problem where we wish to identify the replica of a query within the context window. In general, we can easily contrive natural generalizations of AR and Selective Copy problems that necessitate a more sophisticated attention mechanism (see \cite{poli2024mechanistic}). One such generalization is, given query $q$, we wish to retrieve a general key $k$ (possibly $k\neq q$) and return the value associated with $k$.




\noindent\textbf{N-gram AR.} Our first result shows that a single CAT layer can solve the NAR problem under fairly general conditions.

\begin{theorem}[Solving NAR]\label{thm NAR main} Let $\Fb\in\R^N$ be a causal 1-D convolutional filter of length $N$ and $\norm{\X}$ normalize the rows of a matrix to unit $\ell_2$ norm. Consider a single CAT layer 
$f(\X)=(\X_v\W_v)^{\top}\sft{\X_k\W_k\W_q^{\top}\qb}$
where $\qb$ is the final token of $\X_q$ and $\X_q=\norm{\X\ast\Fb_q} \in \R^{L \times d}$ (same for $\X_k,$). Set $\Fb_q=\Fb$ and $\W_k=\W_q=\sqrt{c}\Iden_d$. Use either
\begin{itemize}
\item \textbf{Value delay:} $\Fb_k=\Fb_q$, $\Fb_v=\dly{-1}$ and $\W_v = 2 \Ib_d$ or,
\item \textbf{Key delay:} $\Fb_k=\dly{1}\ast\Fb_q$, $\Fb_v=\dly{0}$ and $\W_v=\Iden_d$
\end{itemize}
Let $\eps>0$ be the minimum $\ell_2$ distance between two distinct tokens embeddings. For almost all choices of $\Fb$, there is a scalar $c_0>0$ depending on $\Fb$ such that, setting $c=c_0\log(4L/\eps)$, CAT layer solves the NAR problem of Def.~\ref{NAR prob def} for all input sequences up to length $L$.
\end{theorem}

As a corollary, using a simple 1-D convolutional filter on the key embeddings solves the AR problem.

\begin{corollary}[1-D CAT solves AR] Consider a CAT layer employing 1-D convolution on key embeddings with the delay filter $\Fb_{k}=\Db_{1}=[0~1~0~\dots~0]$ and $\Fb_{q}=\Fb_{v}=\Db_{0}$. This model solves AR.
\end{corollary}

\noindent\textbf{Length generalization.} Our next result shows that the global minima of CAT provably exhibit length generalization, thereby shedding light on the empirical benefits of CAT in Figure \ref{fig:lengen_intro}. Concretely, even if we train CAT to solve AR for a fixed context length, the AR capability will generalize to all other context lengths. This result is distinct from Theorem \ref{thm NAR main} because \textbf{it establishes length generalization for all CAT models} that approximately solve the AR problem for a context length, rather than constructing one such solution. The proof is provided in Section \ref{sec len gen new}.

\begin{theorem}[Length generalization]\label{prop len gen new} Let $\Fb_v\in\R^{2W+1}_{+}$ be a convolutional filter from time $t=-W$ to $t=W$ where $W \leq L - 1$. Consider a CAT layer of the form $f(\X)=\X_v^\top \sft{\X\W \x_{L - 1}}$ where $\X \in \R^{L \times d}, \X_v=\X\ast \Fb_v \in \R^{L \times d}$ and $\x_{L - 1}$ is the last token of $\X$ and $\W=\W_k\W_q^\top$. Suppose that token embeddings have unit norm. Consider any model $f=(\W,\Fb_v)$ that can solve the AR problem defined in Def.~\ref{ARprob} up to $\eps$-accuracy on all sequences of length $L\geq 3$. That is, for all $(\X,\y)$ where query $\x_{L - 1}$ repeats twice and $\y$ being the associated value token, we have $\tn{\y-f(\X)}\leq\eps$. Define the minimum embedding distance within vocabulary $\Vc$ as $\Delta=(1-\max_{\ab\neq\bb\in\Vc}(\ab^\top\bb)^2)^{1/2}$ and \shaw{assume that $\Delta > 0$}. There are absolute constants $R_0,R >0$ such that, if $\eps_0:=\eps/\Delta\leq R_0/L$, we have that
\begin{itemize}[leftmargin=20pt]
\item The filter obeys $\tone{\Fb-\dly{-1}}\leq L\eps_0$, which is in line with Theorem \ref{thm NAR main}.
\item Let $\X$ be an input sequence of length $L'$ following Def.~\ref{AR problem def}. Let $\s_\st(\X)\in\R^{L'}$ be the ``golden attention map'' with entries equal to $1/2$ at the positions of the query $\x_{L' - 1}$ and $0$ otherwise. For all such $\X$, the attention map of $f$ obeys $\tone{\sft{\X\W \x_{L'-1}}-\s_\st(\X)}\leq L'\eps_0$.
\item For all $\X$ of length $L'$ following Def.~\ref{AR problem def}, we have that $\tn{\y-f(\X)}\leq R L'\eps_0$.
\end{itemize}
\end{theorem}


Here it worths noting that all CAT models that approximately solve the AR problem ends up learning convolution and attention weights that are consistent with the constructive result of Theorem \ref{thm NAR main}. This simple ``universal solution'' is in contrast to attention-only models where length generalization not only requires standard positional encoding but also additional adjustments to extend the context window of PE \cite{peng2023yarn,chen2023extending}.

Additionally, in Appendix \ref{sec nar len gen}, we generalize the length generalization result to the N-gram AR problem under slightly stronger assumptions, which is specified in Assumption~\ref{assumption bound}. The reader is referred to Proposition \ref{prop len gen nar}. Besides showcasing the value of convolution-attention hybrids, these results also motivate future research into the optimization landscape: Under what conditions gradient methods provably converge to generalizable CAT models, namely those described in Theorem \ref{prop len gen new}? Answers to such questions could build on the recent optimization theoretic results on the transformer/attention models \cite{tian2023joma,tarzanagh2023transformers,deora2023optimization,oymak2023role,li2024mechanics,nichani2024transformers,ildiz2024self,ataee2023max,makkuva2024attention,collins2024context} and extend them to hybrid designs.


\noindent\textbf{Selective Copy.} Our next result shows that, 1-layer CAT model can solve the \emph{unique selective copy} problem. That is, it can provably generate all signal tokens in the correct order as long as the input contains each distinct signal token at most once. Corroborating this, our experiments demonstrate that 1-layer CAT performs on par with or better than alternative architectural choices. The proof is deferred to Section \ref{sel copy proof}.  
\begin{theorem}[Selective Copy]\label{sel copy thm} Consider the setting of Def.~\ref{sel copy def}. There is a 1-layer CAT using exponential-decay query-convolution (i.e.~$F_{q,i}=\rho^i$) and $d=|\Sc|+3$ dimensional token embeddings such that, it outputs all signal tokens in order for all inputs where signal tokens appear uniquely. 
\end{theorem}

Selective Copy problem is distinct from AR in the sense that, it requires a global view of the token positions as the model has to distinguish the order of the distinct signal tokens within the context window. In Theorem \ref{sel copy thm}, we actually describe two ways to achieve this (see appendix for the details): The first option is using an infinitely long convolution $F_{q,i}=\rho^i$ which admits a simple parameterization as a state-space model \cite{gu2021efficiently}. We show that this convolution choice can aggregate all signal tokens in the query embedding while distinguishing their order. This also partly explains how Mamba/SSMs are equally effective in solving Selective Copying. An alternative construction is using a short convolution together with a simple positional encoding. Here, convolution equips the query with local context (specifically the summary of the signal tokens generated so far) and PE provides the global context on the locations of remaining signal tokens. This synergy of PE and short convolution is in line with our real language modeling experiments where CAT with PE outperforms CAT without PE in terms of perplexity as well as length generalization.

\section{{Benefits of Long Convolution for Enabling Sparse-Attention}}\label{sec tradeoff}

So far we have discussed the benefits of short convolutions to equip transformer with local context to solve AR and its variations. During this discussion, we have used dense attention which has exact recall capabilities thanks to its  ability to scan the full context window.
In this section, we ask the following: Can convolution also help mitigate the need for dense attention?
Intuitively, we should be able to tradeoff the accuracy of attention computation with computation. Here, we describe how long convolutions can enable this by effectively summarizing the context window so that we can identify where to attend in (extremely) long-context settings.

\begin{wrapfigure}{r}{0.55\textwidth}
    \centering
    \begin{tikzpicture}
        \node at (0,0) [scale=0.35]{\includegraphics[trim=0 100 20 0, clip]{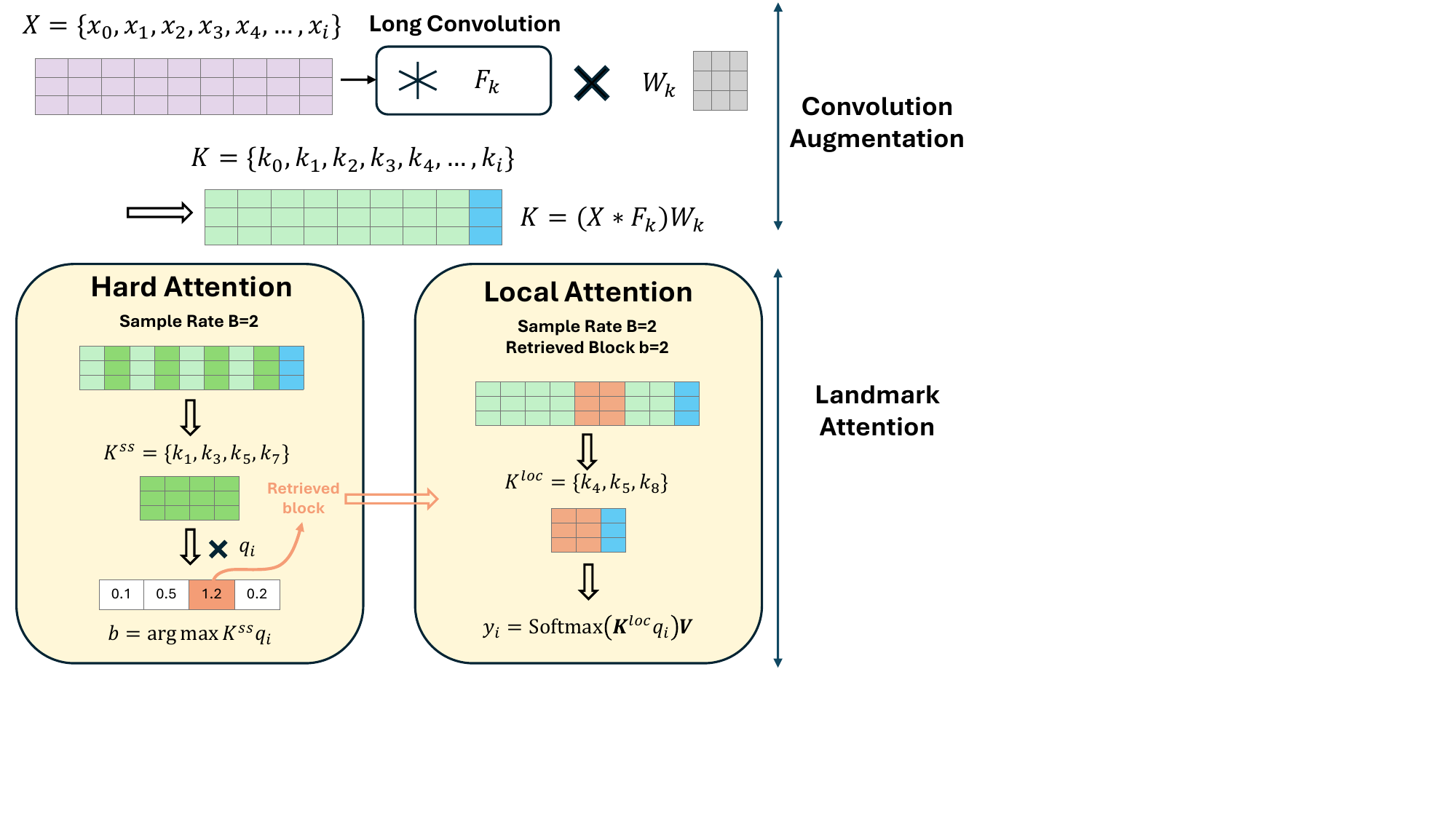}};
    \end{tikzpicture}
    \caption{\small{ Illustration of the Landmark CAT. We first apply long convolution on the input sequence and subsample it to obtain landmark tokens representing individual blocks. 
    Hard Attention computes the similarity between the query and landmarks to retrieve the most relevant block. Local Attention concatenates the retrieved block with the final block containing the query and computes the output token.}}\vspace{-15pt}\label{lcat fig}
\end{wrapfigure}

Specifically, we will prove that, long convolutions (such as SSMs) allow us to utilize sparse attention while retaining (high-probability) recall guarantees. These findings complement the recent research that establish the recall limitations of purely recurrent models \cite{arora2024simple,zoology2023}. Our theory will also shed light on the mechanics of landmark attention \cite{mohtashami2023landmark}. While \cite{mohtashami2023landmark} does not rely on convolution, we will describe how convolution can generate \emph{landmark tokens} by summarizing/hashing the chunks of the context window, and attention can efficiently solve recall by attending only to these summary tokens.


\noindent\textbf{Landmark Convolutional Attention (LCAT):} Figure \ref{lcat fig} describes the LCAT block that apply on input sequence $\X$. Let $\Fb_k\in\R^L$ be the convolutional filter on keys, $B$ be the sampling rate, and $\barL=\lceil L/B\rceil$. Setting $\Kb=(\X\ast \Fb_k)\W_k\in\R^{L\times d}$, we obtain $\Ks\in\R^{\barL\times d}$ by sampling $\Kb$ at every $B$ tokens. Additionally, define $\X_i$ to be the $i$th block of $\X$ of size $B$ spanning tokens $(i-1)B+1$ to $iB$. Let $\Vb=(\Fb_v\ast\X)\W_v$ denote the value embeddings. For a query $\qb_i$ for $i\in[L]$, the LCAT layer outputs:
\begin{align}\label{hcateq}
&\text{(1) Hard Attention:}~~\quad b=\arg\max_{j\neq \lceil i/B\rceil} \Ks\qb_i\tag{LCAT}\\
&\text{(2) Local Attention:}~\quad \y=\sft{\Kbl\qb_i}\Vb^l\quad\text{where}\quad \Kbl=\text{concat}(\Kb_{\lceil i/B\rceil},\Kb_b).\nonumber
\end{align}
Above, \emph{hard attention} phase aims to retrieve the correct block associated to the query. This block is merged with the local block $\lceil i/B\rceil$ that contains the query itself similar to sliding window attention. We then apply \emph{dense local attention} on the concatenated blocks $\Kbl$.

\noindent\textbf{Computational complexity of LCAT:} For a fixed query, \eqref{hcateq} requires $\order{d(L/B+B)}$ computations. This is in contrast to $\order{dL}$ computations of vanilla attention. Choosing a suitable block size (e.g.~$B=\order{\sqrt{L}}$), this model should save up to $\times\sqrt{L}$ in computational savings. Importantly, our theory will highlight the interplay between the embedding dimension $d$ and the allowable acceleration by characterizing the exact performance of \eqref{hcateq} under a random context model.

\begin{definition}[Random Context Model] The query token $\x_L$ occurs twice in the sequence and has unit $\ell_2$ norm. All other tokens of $\X$ are IID and drawn with IID $\Nc(0,\sigma^2/d)$ entries. 
\end{definition}
The following proposition shows that, \eqref{hcateq} will solve AR if and only if $\frac{d}{2B\log\barL}\geq 1+o(1)$.

\begin{proposition}\label{prop long conv1} Recall $\barL=\lceil L/B\rceil$ is the number of blocks. Let $\W_v=2\Iden_d$, $\Fb_v=\Db_{-1}$, and $\W_k=\W_q=\sqrt{c}\cdot\Iden_d$ with $c\rightarrow\infty$. Set key convolution as $F_{k,i}=1$ for $0\leq i<B$ and zero otherwise.\\\textbf{(A)} If $d\geq 2\sigma^2 B(\sqrt{\log\barL}+t)^2$, then \eqref{hcateq} solves AR for fixed $\x_L$ with probability at least $1-3e^{-t^2/4}$.\\\textbf{(B)} Conversely, for any $\eps>0$ there is $C_\eps>0$ as follows: If $\barL\geq C_\eps$ and $d\leq 2\sigma^2 B(\sqrt{(1-\eps)\log\barL}-t)^2$, then \eqref{hcateq} fails to solve AR with the same probability.\\\textbf{(C)} Finally, suppose we wish to solve AR \textbf{uniformly} for all queries $\x_L$ over a subspace $S$. This succeeds with the same probability whenever $d\geq 2\sigma^2 B(\sqrt{\log\barL}+\sqrt{\texttt{dim}(S)}+t)^2$. 
\end{proposition}

Figure \ref{fig_vision} corroborates the predictive accuracy of Proposition~\ref{prop long conv1}: As the block size increases, the embedding dimension to maintain success of AR grows approximately linearly. One can expand on this proposition in two directions. Firstly, a fundamental bottleneck in \eqref{hcateq} is the requirement $d\gtrsim B\log\barL$. This arises from a \emph{memory-recall tradeoff} \cite{arora2024simple,jelassi2024repeat} as we are summarizing the information of block $\X_i$ of length $B$ through its landmark token.
However, once this requirement is satisfied, the model can identify the correct block in $\order{\barL}$ cost. To avoid paying the additional $\order{B}$ cost of local attention, we could apply the LCAT approach hierarchically within the selected block to reduce the compute cost to $d(\barL+\log B)$ per token. The dominant term $d\barL$ captures the recall capacity of the LCAT model: Consistent with our theorem and lower bounds of \cite{arora2024simple}, for AR to succeed, we need \[ 
\text{recall\_capacity}=d\barL\geq L=\text{required\_memory}
\]
\begin{wrapfigure}{r}{0.5\textwidth}
    \centering 
    \includegraphics[width=0.45\textwidth]{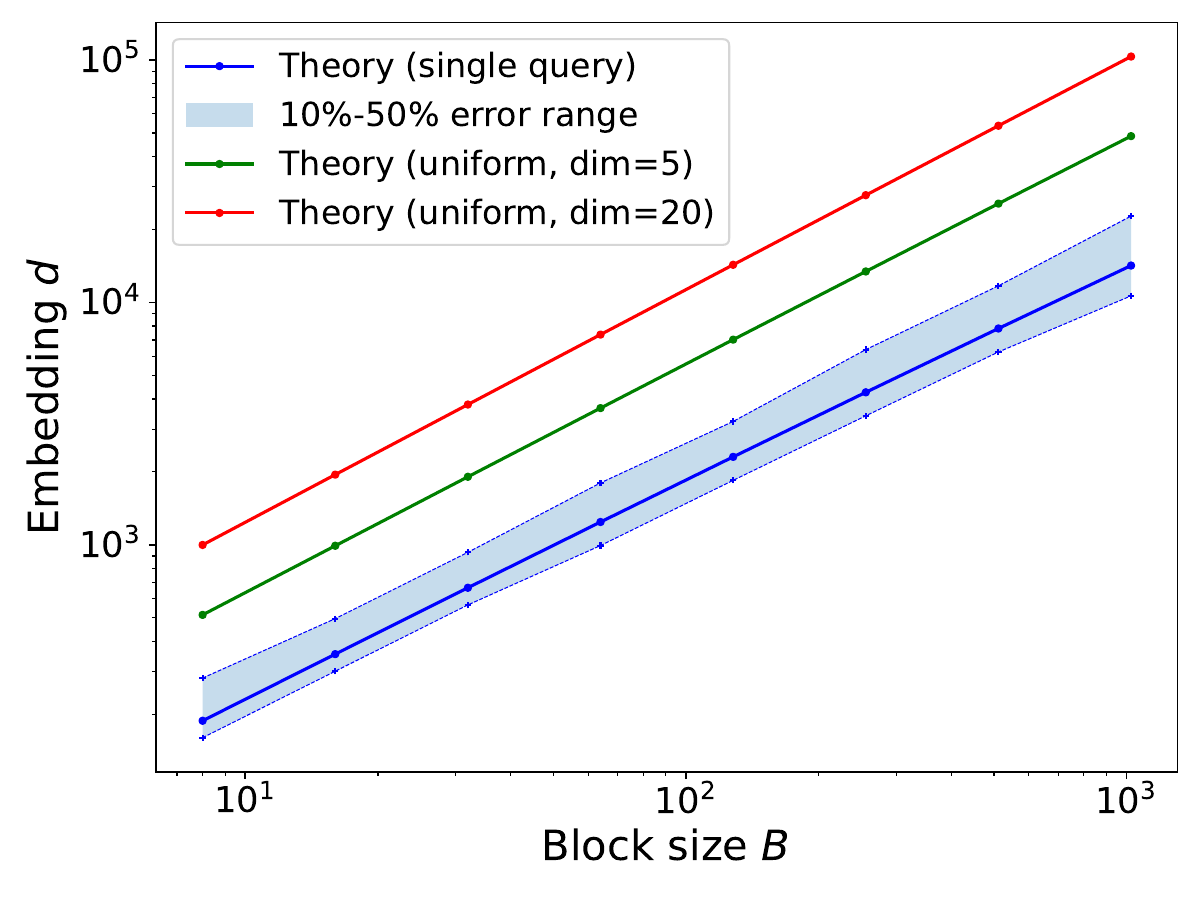}
    \caption{\small{Behavior of the embedding dimension as a function of block size for context length $L=2^{20}\approx 1$ million (noise level $\sigma^2=1)$. Shaded region highlights te range of $d$ that exhibits 10\%-50\% empirical success. Proposition \ref{prop long conv1} accurately captures the empirical behavior. For the success of uniform AR, we need larger $d$ as the dimension of the query space $S$ grows.}}
    \label{fig_vision}
    \vspace{-35pt}
\end{wrapfigure}

Secondly, Proposition \eqref{prop long conv1} chooses a particular long convolution where landmarks become the mean of the input tokens within the block. In practice, we can use a state-space model \cite{gu2022parameterization} to parameterize convolution efficiently. A particular SSM choice of state dimension 1 is simply using exponential smoothing. This yields the following SSM variant of Proposition \ref{prop long conv1}.


\begin{proposition}\label{prop long conv2} Consider the setting of Proposition \ref{prop long conv1} with the exponential smoothing filter $\Fb_i=\rho^i$ for $i\geq 0$. Set $\rho=e^{-1/B}$ so that $\rho^B=e^{-1}$. Suppose $d\geq 50B(\sqrt{\log\barL}+t)^2$. Then, \eqref{hcateq} solves AR with probability at least $1-3e^{-t^2/4}$.
\end{proposition}

Above, we fixed the decay rate $\rho$ for exposition purposes. More generally, any $\rho$ choice with an effective context size of $\order{B}$ would result in similar guarantee.


\section{Experiments}\label{sec:experiments}
\subsection{Model Evaluation on N-gram AR and Length Generalization Capability}\label{sec:syn_exp}

For the synthetic experiments on associative recall problems, we employ the CAT architecture as detailed in Section~\ref{sec:methodology}. We utilize convolution kernels with a width of $W=3$ and explore model embedding sizes of $d = 32$, $64$, and $128$ across MQAR and MQNAR problems to assess the impact of model dimension on performance. In addition to the standard attention mechanism, we introduce a perturbation strategy by implementing linear attention on the convoluted $Q$, $K$, and $V$ embeddings, referred to as LinCAT. We adhere strictly to the parameters set by \cite{zoology2023}. Our experimental setup and code are available on GitHub\footnote{\url{https://github.com/umich-sota/CAT}}. More detailed information on the training setup can be found in Section~\ref{app:exp_setup} including the data generation and hyperparameters. For reporting results, we conduct each experiment three times and present the maximum accuracy achieved across these runs, aligning with the methodologies of \cite{zoology2023} and \cite{arora2024simple}.

As illustrated in Fig.~\ref{fig:train_syn}, the CAT model consistently outperforms all baseline models across a range of sequence lengths and model dimensions. Notably, both Mamba and Based models exhibit improved performance as the model dimension increases, particularly with shorter sequence lengths. This improvement is due to the memory-recall tradeoff~\cite{arora2024simple} where models store and recall sequence information more as their dimensionalities expand. In contrast, thanks to the short convolution, the \emph{single-layer} CAT model maintains 100\% accuracy across all experimental settings, aligned with our theorem~\ref{thm NAR main}. Interestingly, aside from CAT, Mamba is the only model demonstrating the potential to effectively address the MQAR task within a single-layer network architecture. We will discuss this observation in further detail in Section~\ref{app:syn_exp}.

\begin{figure}[t]
    \centering
    \begin{subfigure}{\textwidth}
    \begin{tikzpicture}
        \node at (0,0) [scale=0.65] {\includegraphics[trim=15 10 0 10, clip]{figs/conv_1gram.pdf}};
    \end{tikzpicture}\label{fig:1gram_2layer}
    \end{subfigure}\vspace{-7pt}
    \begin{subfigure}{\textwidth}
        \begin{tikzpicture}
            \node at (0,0) [scale=0.65] {\includegraphics[trim=15 10 0 10, clip]{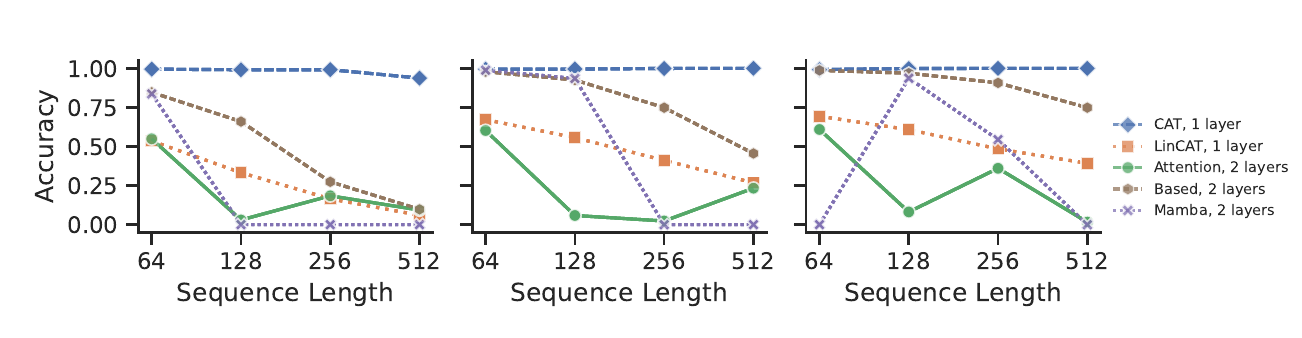}};
        \end{tikzpicture}\label{fig:2gram_2layer}
        \end{subfigure}
    \caption{\small{Evaluation of models on MQAR and MQNAR tasks with varying model dimensions and sequence lengths. Model dimensions are \textbf{32, 64, 128} for each column of the figures, from left to right. \textbf{Top:} Models trained on the MQAR setup. \textbf{Bottom:} Models trained on the MQNAR setup. Note that CAT models employ a single-layer architecture, whereas all other models utilize two layers. Refer to Section~\ref{sec:syn_exp} for detailed setup descriptions.}\label{fig:train_syn}
    }\vspace{-15pt}
    
\end{figure}

\begin{figure}[t]
    \centering
    \begin{subfigure}{\textwidth}
    \begin{tikzpicture}
        \node at (0,0) [scale=0.65] {\includegraphics[trim=15 10 0 10, clip]{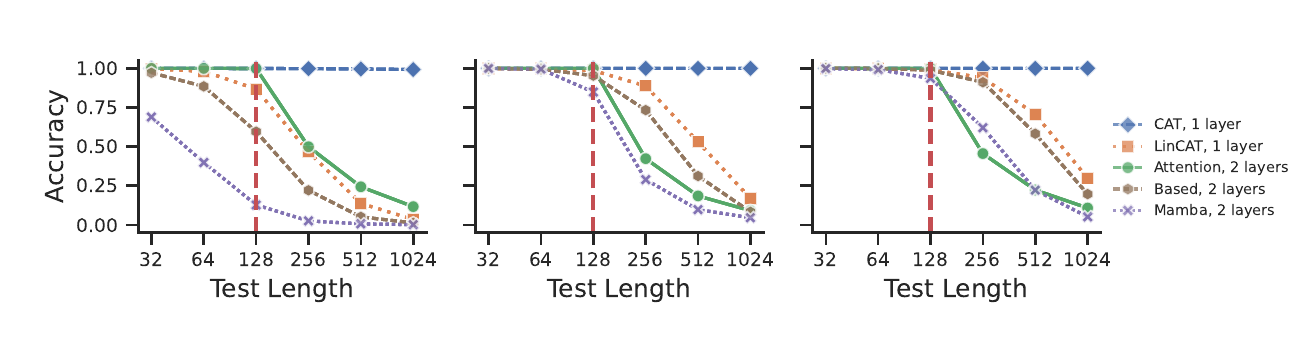}};
    \end{tikzpicture}\label{fig:1gram_2layer_lengen}
    \end{subfigure}\vspace{-9pt}
    \begin{subfigure}{\textwidth}
        \begin{tikzpicture}
            \node at (0,0) [scale=0.65] {\includegraphics[trim=15 10 0 10, clip]{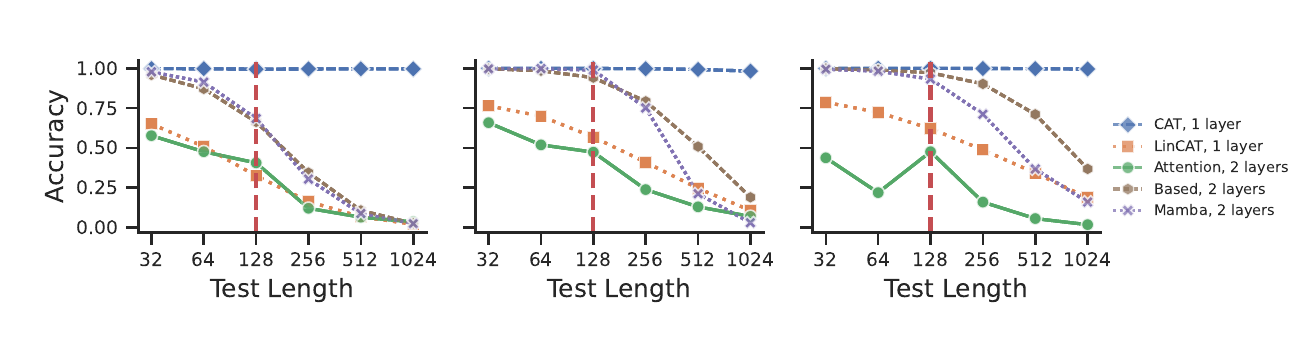}};
        \end{tikzpicture}\label{fig:2gram_2layer_lengen}
        \end{subfigure}
    \caption{\small{Evaluation of models on length generalization. Model dimensions are \textbf{32, 64, 128} for each column of the figures, from left to right. The models are trained with sequence length 128 (vertical red dashed lines) and tested on varying test length. \textbf{Top:} Models trained on the MQAR. \textbf{Bottom:} Models trained on the MQNAR. Note that CAT models establish length generalization aligned with Theorem~\ref{prop len gen new}}
    .\label{fig:lengen}
    }
\end{figure}

\textbf{Evaluation of Length Generalization.}
  In Fig.~\ref{fig:lengen}, we train models with 128 sequence length (the vertical red dashed line) and evaluate their performance on varying sequence lengths from 32 to 1,024. Fig.~\ref{fig:lengen} shows the results of length generalization, which is aligned with our Theorem~\ref{prop len gen new}: CAT models maintain 100\% accuracy while all other models exhibit a sharp decline in performance as the sequence length increases. This decrease is due to the increased demand of recall which requires the model to store and retrieve more information as the sequence length grows. The CAT model, however, is able to maintain its performance by leveraging the convolutional filters to shift the context and retrieve the necessary information. 
  Remarkably, in Fig.~\ref{fig:train_syn}, we observe non-monotonic accuracy behavior for Mamba and Attention-only models as a function of sequence length. This is due to the fact that these models are more sensitive and harder to optimize in AR problems. In Fig \ref{fig:lengen}, we used a denser hyperparameter grid and more trials to ensure smoother curves with better reproducibility.


\begin{table}[t]\caption{\small{Experiment results for model pretraining. $^*$ are results from~\cite{yang2023gated}, which uses a same dataset and training procedure as ours. We use the same hyperparameters as~\cite{yang2023gated} for fair comparison. For perplexity, lower is better, and for accuracy, higher is better. The average accuracy in last column is calculated by averaging the accuracy across all tasks but excluding the perplexity tasks. The best and second best results are highlighted in boldface and underline, respectively.} \label{tab:real_nlp}}
    \centering
    
    \resizebox{\textwidth}{!}{%
    \begin{tabular}{@{}lccc|ccccccc|c@{}}
    \toprule
    {\thead{Model}} & {\thead{Wikitext \\ ppl$\downarrow$}} & {\thead{Lambada\_std \\ ppl$\downarrow$}} & {\thead{Lambada\_openai \\ ppl$\downarrow$}} &
    {\thead{Lambada\_std \\ acc$\uparrow$}} & {\thead{Lambada\_openai \\ acc$\uparrow$}} & {\thead{Piqa \\ acc$\uparrow$}} & {\thead{Hella \\ acc\_norm$\uparrow$}} &
    {\thead{Winogrande \\ acc$\uparrow$}} & {\thead{Arc-E \\acc$\uparrow$}} & {\thead{Arc-C \\ acc\_norm$\uparrow$}} & {\thead{Avg \\ Acc$\uparrow$}} \\ \midrule


    Pythia  & 27.410  & 74.663  & 34.023 & 0.281 & 0.343 & 0.651 & 0.355 & \underline{0.529} & 0.443 & 0.235 & 0.405 \\
    CAT, no PE & 29.216 & 86.318 & 42.260 & 0.266 & 0.321 & 0.640 & 0.339 & 0.515 & 0.436 & 0.237 & 0.393 \\
    CAT, RoPE & \underline{26.776} & 65.423 & 38.557 & 0.288 & 0.341 & \underline{0.654} & \underline{0.362} & 0.507 & 0.461 & 0.239 & 0.407 \\
    MH-CAT, no PE & 27.417 & \underline{58.959} & \underline{32.822} & \underline{0.296} & \underline{0.355} & 0.644 & 0.352 & \textbf{0.531} & 0.460 & \underline{0.240} & \underline{0.411} \\
    MH-CAT, RoPE &\textbf{25.858}&\textbf{47.593}&\textbf{28.273}&\textbf{0.330}&\textbf{0.377}&\textbf{0.662}&\textbf{0.376}&0.512&\textbf{0.466}&0.231&\textbf{0.422} \\
    
    TF++~\cite{touvron2023llama}$^*$ & 28.390 & NA & 42.690 & NA &0.310 &0.633 &0.340&0.504&0.445&\textbf{0.242} & NA\\
    Mamba~\cite{gu2023mamba}$^*$ &28.390 & NA &39.660 & NA &0.306 & 0.650 & 0.354 &0.501 &\underline{0.463} & 0.236 & NA\\
    GLA~\cite{yang2023gated}$^*$ & 28.650 & NA &  43.350 & NA & 0.303 &0.648 &0.345 & 0.514 & 0.451 & 0.227 & NA \\
    
    \bottomrule
    \end{tabular}%
    }
    \end{table}

\subsection{Evaluations on Language Modeling}\label{sec:nlp_exp}

Based on the outcomes from the synthetic experiments, we further explore the efficacy of the CAT model in real-world NLP tasks by integrating a 1D CAT structure into the Pythia~\cite{biderman2023pythia} framework. We pretrain the modified 370M-parameter model on the SlimPajama~\cite{soboleva2023slimpajama} dataset, involving 15 billion tokens. We then assess the model on a variety of downstream zero-shot tasks, including Wikitext, Lambada, Piqa, Hella, Winogrande, Arc-E, and Arc-C, a methodology commonly used in the field to evaluate generalization capabilities across diverse tasks~\cite{biderman2023pythia, gu2023mamba, zoology2023, arora2024simple}. The findings are compiled in Table~\ref{tab:real_nlp}.

In this series of experiments, the CAT model is trained in two variants: one incorporating rotary positional embedding~\cite{su2024roformer} (PE) and another without positional embedding (noPE). We observe that the CAT model with PE not only consistently outperforms the Pythia model but also achieves performance better than state-of-the-art models, including Mamba~\cite{gu2023mamba}, TF++~\cite{touvron2023llama}, and GLA~\cite{yang2023gated}. Notably, the CAT model secures a superior perplexity gain compared to the standard model while maintaining a similar level of parameters.

Regarding the noPE variant, training a Pythia model without positional encoding leads directly to divergence and extremely large losses during training, affirming the critical role of positional encoding in enabling standard transformer models to learn and converge. Intriguingly, despite the absence of positional encoding, the CAT model still performs competitively with the leading models. This suggests that the convolutional structure in the CAT model effectively captures positional information within the data. We conjecture that the short convolutions provide positional information for neighboring tokens, while the deep multi-layer network structure hierarchically aggregates this information to establish long-range positional information.

This observation aligns with our synthetic experiment results, where the CAT model demonstrated the capability to handle the AR task without positional encoding. These insights indicate that the convolutional structure could potentially replace positional encoding, which might benefit length extrapolation and generalization in the model. This offers a promising direction for further model design and optimization in the field of NLP.

$\bullet$\textbf{Length Generalization} Figure~\ref{fig:lengen_intro} presents the results from a length generalization experiment with the CAT model, in which we trained the model on sequences of length 2,048 and assessed its zero-shot performance on the Wikitext dataset across varying test sequence lengths. As a baseline in our analysis, we implemented position interpolation (PI)~\cite{chen2023extending} and YaRN~\cite{peng2023yarn} tempreture scaling on RoPE models, including CAT/MH-CAT RoPE, to facilitate length generalization. The results indicate that among the three RoPE models examined, the CAT model consistently demonstrates excellent performance across all test sequence lengths. In contrast, the Pythia model exhibits a sharp decline in performance as the sequence length increases.  We suggest that is due to the additional positional embeddings introduced by PI that was absent during the training phase.  Despite this, CAT models proficiently manage the relative positioning of tokens (especially overcome the new positional embeddings by leveraging convolution information), which significantly boosts its ability for length generalization. Additionally, the CAT model without PE is superior to the Pythia model with RoPE, suggesting the effectiveness of the convolutional structure within the CAT model in capturing essential positional data in length extrapolation.

\subsection{Model Evaluation on Selective Copying}\label{sec:syn_selective}

Fig.~\ref{fig:selective_copying} displays the selective copying results for 1-layer and 2-layer models. We train these models across a variety of model dimensions and sequence lengths. The models are required to copy 16 signal tokens from the input sequence and output them in the correct order. We observe that all 2-layer models perform well and show overlapping results, except for LinCAT. Among the 1-layer models, CAT and Mamba achieve nearly 100\% accuracy, while the performance of other models is lower. These results are consistent with Theorem~\ref{sel copy thm} and demonstrate that the 1-layer CAT model can solve the selective copying problem without repetitions.

\begin{figure}[t]
    \centering
    \begin{subfigure}{\textwidth}
        \centering
        \begin{tikzpicture}
            \node at (0,0) [scale=0.65] {\includegraphics[trim=15 10 0 10, clip]{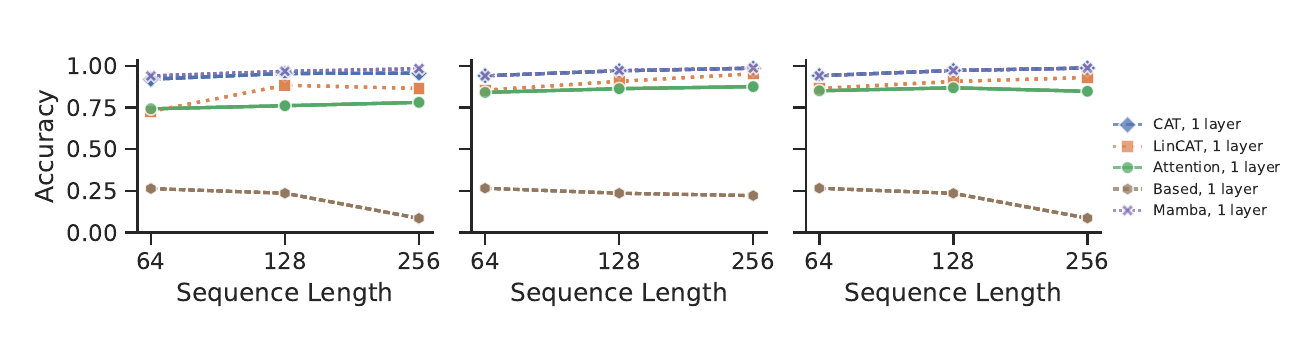}};
        \end{tikzpicture}
    \end{subfigure}\vspace{-7pt}
    \begin{subfigure}{\textwidth}
        \centering
        \begin{tikzpicture}
            \node at (0,0) [scale=0.65] {\includegraphics[trim=15 10 0 10, clip]{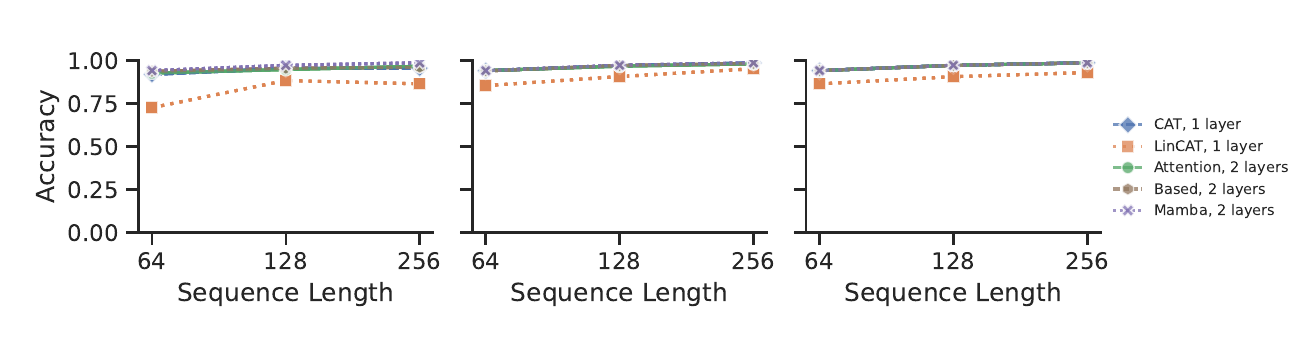}};
        \end{tikzpicture}
        \end{subfigure}
    \caption{\small{Evaluation of models on selective copying tasks with varying model dimensions and sequence lengths. Model dimensions are \textbf{32, 64, 128} for each column of the figures, from left to right. \textbf{Top:} Models trained on 1-layer architectures. \textbf{Bottom:} Models trained on 2-layer architectures. Note on 1-layer experiment, CAT and Mamba achieve nearly 100\% and their curves are overlapped. 
    }\label{fig:selective_copying}
    }
    
\end{figure}

\section{Discussion}\label{sec:limitation}
\vspace{-5pt}

In this work, we have examined the synergy between the attention and convolution mechanisms by introducing Convolution-Augmented Attention where K/Q/V embeddings are equipped with convolution. We have shown that CAT layer enjoys strong theoretical guarantees when it comes to AR and copying tasks with length generalization and also described insightful tradeoffs between the need for attention and convolution. Importantly, real experiments confirm the benefit of CAT model both in accuracy and in length generalization. Ultimately, we believe this work as well as the related recent literature \cite{gu2023mamba,arora2024simple,poli2024mechanistic,dao2024transformers} contributes to stronger design principles for the next generation of (hybrid) architectures. 

\noindent\textbf{Limitations and future work.} This work has a few shortcomings. We have only focused on pretraining. However, Fig.~\ref{fig:lengen_intro} shows the potential of CAT in finetuning as a future direction. While K/Q convolution helps in theoretical constructions for N-gram AR, in real experiments, they don't provide noticeable performance benefits. We suspect that K/Q convolution might be \emph{diluting} the attention scores and incorporating normalization or better parameterization can address this issue. An important parameterization to explore is replacing the short convolutions within CAT with SSMs. Finally, Section \ref{sec tradeoff} introduced Landmark CAT as a sparse attention strategy. It would be interesting to evaluate this proposal on real language modeling tasks.

\subsection*{Acknowledgements}

This work was supported in part by the National Science Foundation grants CCF-2046816, CCF-2403075, the Office of Naval Research award N000142412289, and gifts by Open Philanthropy and Google Research. 

\bibliography{refs}
\bibliographystyle{plain}

\clearpage
\appendix
\section{Detailed Experiment Setup}\label{app:exp_setup}

\subsection{Associative Recall Experiments}
We first introduce the training setup for the synthetic experiments. In our MQAR and MQNAR experiments, we create a dataset with a vocabulary size of 8,092 to ensure that the vocabulary replicates the scope of real language data. The dataset is constructed as described in Sec.~\ref{sec:MQAR}, with varying sequence lengths $L$ of 64, 128, and 256, and 512. Specifically, we formulate the dataset in the form of key-value pairs accompanied by multiple queries where the keys are unique within each sequence. For each example, we initially select $k$ keys from the vocabulary without replacement and subsequently draw the values from the remaining vocabulary. We then randomly shuffle the keys and associated values to form the input sequence. The number of queries is set to match $k$, ensuring each key in the sequence is queried. It should be noted that while the keys are unique within a single example, they may be repeated across different examples. For sequence lengths of $L$ = 64, 128, 256, and 512, we set $k$ = 16, 32, 64, and 128 respectively, indicating that the number of keys and queries scales with the sequence length, thus increasing the task complexity. We generate 100,000 training examples and 3,000 testing examples for each of the sequence lengths. For NAR experiment, we primarily focus on $N=2$ to evaluate the performance. We construct the dataset similarly to the MQAR task with sequence lengths of 64, 128, and 256. Consequently, the number of keys and queries is reduced to $k=10, 20, 40$ respectively, to accommodate the larger $N$. We generate 200,000 training examples and 3,000 testing examples for each sequence length.

For the training, we adhere strictly to the parameters set by \cite{zoology2023}, and their experimental setup and code, using learning rate sweep among $0.001, 0.01, 0.1$ and train the model for 64 epoches. The maximum accuracy achieved across these learning rate is reported. 

We remark that for the length generalization experiments, we sweep the learning rate among $0.001, 0.003, 0.01, 0.03, 0.1$ and report the maximum accuracy over 5 runs to ensure the robustness and reproducibility of the results.

\subsection{Language Modeling Experiments}

For the language modeling experiments, we exactly follow the setup from~\cite{yang2023gated}. For the length generalization experiment, we train the model on sequences of length 2,048 and assess its zero-shot performance on the Wikitext dataset across varying test sequence lengths. 




\section{Additional Experiments}\label{app:syn_exp}

We conduct additional Experiments, Fig.~\ref{fig:1gram_1layer} and~\ref{fig:2gram_1layer} shows the result of 1-layer models on 1-gram and 2-gram MQNAR tasks with varying hidden sizes and sequence lengths. The model dimension is set to 32, 64, and 128 for each column of the figures, from left to right. 
 All other models perform much worse compare to their 2-layer counterparts. Fig.~\ref{fig:1gram_1layer_lengen} and~\ref{fig:2gram_1layer_lengen} show the length generalization results of 1-layer models on 1-gram and 2-gram MQNAR tasks. The results are consistent with the 2-layer models. 
\begin{figure}[t]
    \centering
    \begin{tikzpicture}
        \node at (0,0) [scale=0.65] {\includegraphics[trim=15 10 0 10, clip]{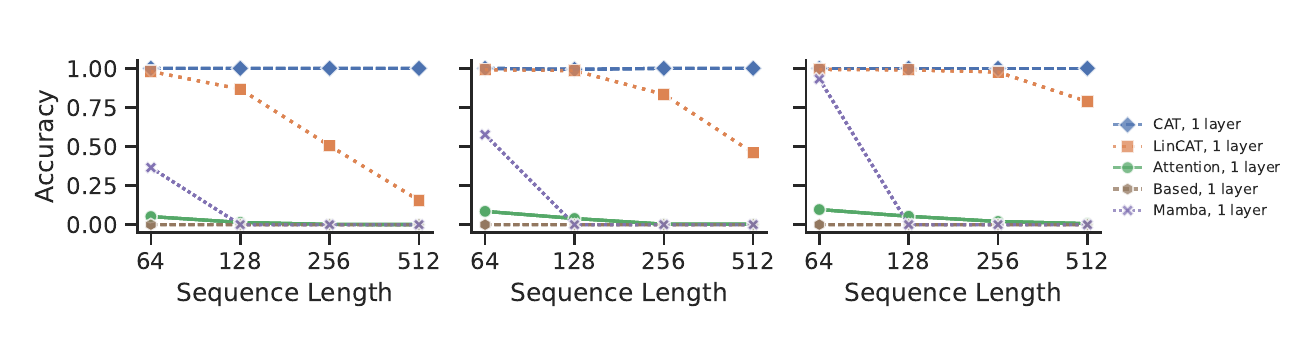}};
    \end{tikzpicture}\caption{Performance of 1-layer models on MQAR tasks with varying model dimension and sequence length. Noted that all models are trained using 1-layer architecture. 
    }\label{fig:1gram_1layer}
\end{figure}

\begin{figure}
    \centering
    \begin{tikzpicture}
        \node at (0,0) [scale=0.65] {\includegraphics[trim=15 10 0 10, clip]{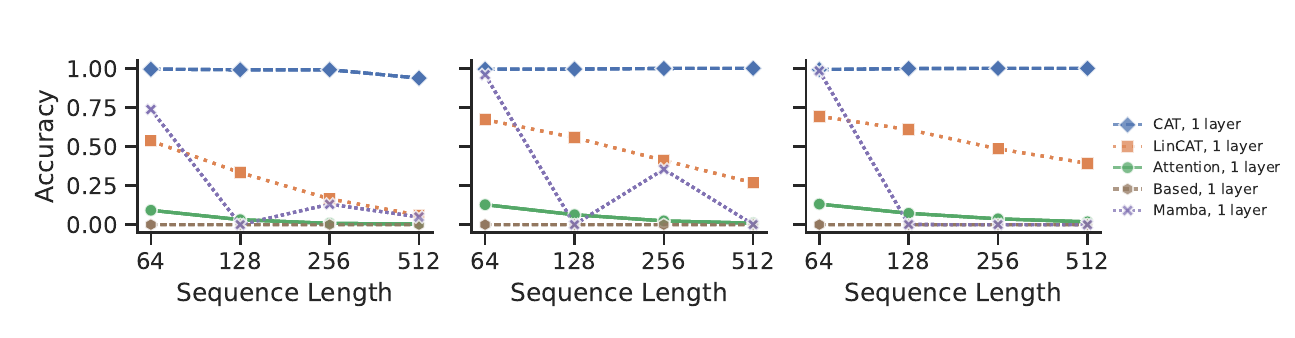}};
    \end{tikzpicture}\caption{Performance of 1-layer CAT models on 2-gram MQNAR tasks with varying hidden sizes and sequence length. All models are trained using 1-layer architecture. 
    }\label{fig:2gram_1layer}
\end{figure}

\begin{figure}[t]
    \centering
    \begin{tikzpicture}
        \node at (0,0) [scale=0.65] {\includegraphics[trim=15 10 0 10, clip]{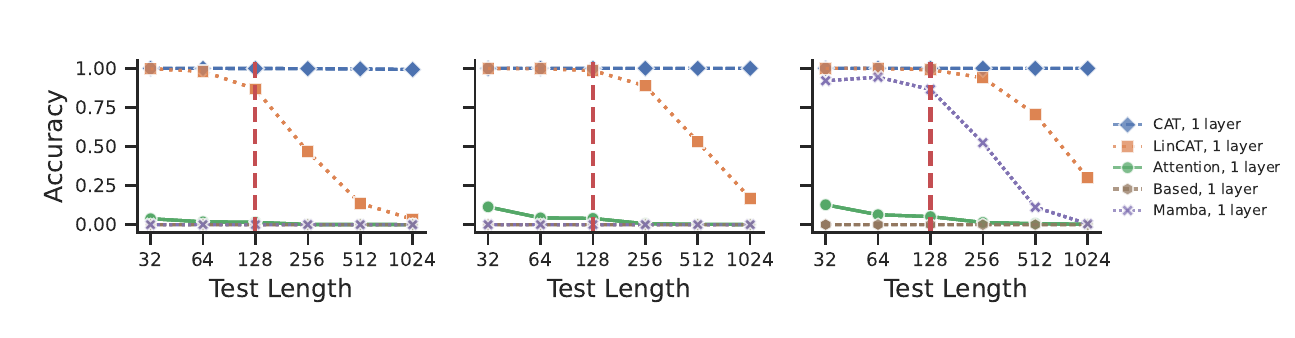}};
    \end{tikzpicture}\caption{ 1-gram Length generalization
    }\label{fig:1gram_1layer_lengen}
\end{figure}

\begin{figure}
    \centering
    \begin{tikzpicture}
        \node at (0,0) [scale=0.65] {\includegraphics[trim=15 10 0 10, clip]{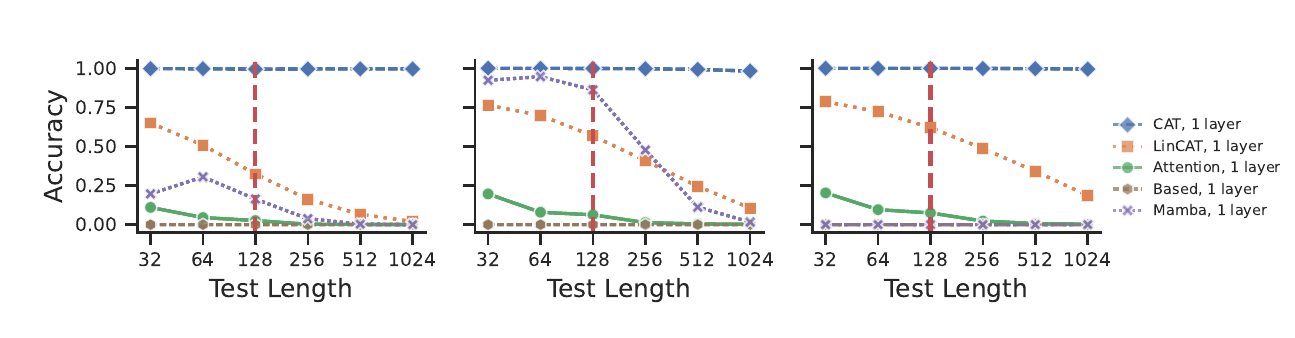}};
    \end{tikzpicture}\caption{ 2-gram Length generalization
    }\label{fig:2gram_1layer_lengen}
\end{figure}



\section{Proofs on Associative Recall and Selective Copying}

\subsection{Proof of Theorem \ref{thm NAR main}}
\begin{proof} Given an $N$-gram $\Z\in\R^{N\times d}$, let us define $\s(\Z)=\norm{\sum_{i\in [N]} F_{N-i}\z_i}$ to be its signature. We will first show that for almost all $\Fb$, each $N$-gram admits a unique signature. To see this, let $\A,\Z\in\R^{N\times d}$ be two distinct $N$-grams. Let us write the difference between their signatures as a correlation coefficient. Set $\s'(\Z)= \sum_{i \in[N]} F_{N-i}\z_i$. Note that if $\s(\A)=\s(\Z)$, we would have the following function of $\Fb$ that arises from correlation coefficient as zero:
\[ 
g_{\A,\Z}(\Fb)=(\s'(\Z)^\top \s'(\A))^2-\tn{\s'(\Z)}^2\tn{\s'(\A)}^2.
\]
Now, observe that $g$ is a fourth-order polynomial of the entries of $\Fb\in\R^N$ and we can expand $g(\Fb)$ further as follows
\begin{align}
g(\Fb)&=(\sum_{i\in[N]}\sum_{j\in[N]} F_{N-i}F_{N-j}\ab_i^\top\z_j)^2-\tn{\sum_{i\in[N]}F_{N-i}\ab_i}^2\tn{\sum_{i\in[N]}F_{N-i}\z_i}^2.
\end{align}
Above, let $c_i$ be the coefficient of the fourth moment term $F_{N-i}^4$. Note that
\[ 
c_i=(\ab_i^\top\z_i)^2-\tn{\ab_i}^2\tn{\z_i}^2.
\]
Since $\A\neq\Z$, there exists $i\in[N]$ such that $\ab_i\neq \z_i$. This implies that $c_i\neq 0$ and $g(\Fb)$ is a nonzero polynomial. As a result, $g(\Fb)\neq 0$ almost everywhere implying the same for $\s(\Z)\neq \s(\A)$. Since there are finitely many $N$-grams, repeating the same argument for all $N$-gram pairs, we find that all $N$-gram signatures are unique for almost all $\Fb$.

Next, suppose we have an $\Fb$ resulting in unique signatures. We will prove the ability of CAT layer to solve the N-AR problem. Consider an arbitrary sequence $\X$ and denote the last $N$ tokens by $\Z$. Let $\X_*=\norm{\X\ast\Fb}$ be the convolved sequence and let $\qb$ be the final token of $\X_*$. By assumption, $\qb$ repeats exactly twice in the sequence. Let $\alpha$ be the position of the $\qb$ in the sequence. By definition, the target token $\vb = \x_{\alpha + 1}$. Let $\Ic_i\in\R^L$ be the indicator function that has $1$ at position $i$ and $0$ everywhere else. Since all $N$-grams are unique and their signatures have unit norm, we have that
\begin{align}\label{core map}
\lim_{c\rightarrow\infty}\sft{c\X_*\qb}=\s_*(\X):=\frac{\Ic_L+\Ic_\alpha}{2}.
\end{align}
Above we use the standard fact that softmax will saturate at the top entry as the inverse-temperature goes to infinity. For the purposes of length generalization, we provide the precise temperature requirement. Let $\ab,\bb$ be two vectors in the normalized $N$-gram token set $\Sc_N$ (the set of tokens obtained after convolving with $\Fb$). Over all such $\ab,\bb$, define the minimum cosine distance to be
\[ 
\Delta =1-\max_{\ab\neq \bb\in \Sc_N}\ab^\top\bb.
\] 
Given sequence $\X_*$, using the worst case likelihood ratios of $e^{-\Delta}$ between the two $\qb$-tokens vs the remaining $L-2$ non-$\qb$ $N$-grams tokens, for any $\X$, we have that
\begin{align}\label{basic arg}
\tone{\text{map}(\X,c)-\s_*(\X)}=\tone{\sft{c\X_*\qb}-\s_*(\X)}\leq \frac{2(L-2)e^{-c\Delta}}{2+(L-2)e^{-c\Delta}}.
\end{align}
To make the right hand side $\leq \eps/2$ for all (admissible) sequences $\X$ of length at most $L$, we need $2(L-2)e^{-c\Delta}\leq \eps$ which implies $c\geq \Delta^{-1}\log(\frac{2(L-2)}{\eps})$.

\noindent\textbf{Value delay.} For value delay, we will use \eqref{basic arg} as key and query embeddings use the same filter. Let $\X_v=2\cdot \X\ast\dly{-1}$. Using the fact that rows of $\X_v$ are unit norm, for $c\geq \Delta^{-1}\log(\frac{2(L-2)}{\eps})$
\[ 
\tn{\X_v^\top \text{map}(\X,c)-\X_v^\top \s_*(\X)}\leq 2\tone{\sft{c\X_*\qb}-\s_*(\X)}\leq \eps.
\] 
Next, note that 
\[ 
\X_v^\top  \s_*(\X)=\X_v^\top (\frac{\Ic_L+\Ic_\alpha}{2})=\frac{\vb_{\alpha}+\vb_{L}}{2}.
\]
Now observe that, thanks to $-1$ delay, $\vb_{\alpha}=2\x_{\alpha+1}$ and $\vb_L=2\x_{L+1}=0$ resulting in $\lim_{c\rightarrow\infty}\X_v^\top \text{map}(\X,c)=\x_{\alpha+1}$. Combining the above results, we find that $\tn{\Vb^\top \text{map}(\X,c)-\vb}\leq \eps$ for all $\X$.

\noindent\textbf{Key delay.} In this scenario, we are delaying $\X_*$ forward by one. Because of this, we have $\X_k=\X_*\ast\dly{1}$ and, within $\X_k$, $\qb$ appears in positions $\alpha+1$ and $L+1$. Since the latter is out of bounds, repeating the argument \eqref{basic arg} and defining $\text{map}(\X,c):=\sft{c\X_k\qb}$, for any sequence $\X$, we find that
\[ 
\tone{\sft{c\X_k\qb}-\Ic_{\alpha+1}}\leq \frac{2(L-1)e^{-c\Delta}}{1+(L-1)e^{-c\Delta}}.
\]
Similarly, the right hand side is upper bounded by $\eps$, whenever $c\geq \Delta^{-1}\log(\frac{2(L-1)}{\eps})$.

To conclude, using the fact that tokens are unit norm and the target value vector is $\vb=\x_{\alpha+1}$, for any $\X$, we obtain
\[
\tn{\X^\top \text{map}(\X,c)-\vb}\leq \tone{\sft{c\X_*\qb}-\Ic_{\alpha+1}}\leq \eps,
\] 
completing the proof that $\tn{\X^\top \text{map}(\X,c)-\vb}$ for all $\X$ of length at most $L$.

\textbf{Concluding the proof of the theorem statement.} So far, we have concluded that, for all input sequences $\X$, CAT layer output guarantees $\tn{f(\X)-\vb}< \eps_0$ where $\vb$ is the target value token and $\eps_0$ is under our control by choosing $c=\Delta^{-1}\log(2L/\eps_0)$. Since we assume the minimum distance between distinct token embeddings are $\eps$, to accurately and uniquely decode the target $\vb$, we choose $\eps_0=\eps/2$ and apply nearest neighbor on $f(\X)$ to conclude.
\end{proof}
\subsection{Proof of Theorem \ref{prop len gen new}}\label{sec len gen new}

{In this section, we will use the shorthand $\Fb$ to denote the value filter $\Fb_v$} for notational simplicity. {Recall that $R_0>0$ is an absolute constant throughout the proof. Finally, the constant $R$ used in Theorem \ref{prop len gen new}'s statement will be subsumed within the $\order{\cdot}$ notation below.}

\begin{lemma} \label{lemma filter ar} Consider the same setting in Theorem~\ref{prop len gen new}. For any $f=(\W,\Fb)$ that can solve the AR problem {defined in Def.~\ref{ARprob}} up to $\eps$-accuracy on all sequences of length $L \geq 3$, if \shaw{$\eps_0:=\eps/\Delta\leq 1/8$}, we have that
    \begin{align}
    &\tone{\Fb-\Db_{-1}}\leq \order{W\eps_0(1+L\eps_0)+L\eps_0}\leq \order{L\eps_0(1+W\eps_0)}\label{main bounds}\\ 
    &\tone{\Fb_{\geq 0}}=\sum_{i=0}^{W}F_{i}\leq \order{\eps_0(1+L\eps_0)}
    \end{align}
    where we use $\order{\cdot}$ notation to denote an upper bound up to a constant i.e. for some absolute $r >0$, $\order{x}\leq r \cdot x$. 
    Moreover, let $\qb$ be a token within vocabulary $\Vc$ 
    and $\vb$ be the top query not equal to $\qb$ that maximizes the similarity $\vb^{\top}\W\qb$ i.e.~$\vb=\arg\max_{\x\in\Vc,\x\neq\qb}\x^{\top}\W\qb$, we have
    \begin{align}
        o_q=s_v/s_q\leq \Gamma=\frac{\eps_0}{1-4\eps_0}=\order{\eps_0}
    \end{align}
    where $s_q$ and $s_v$ are the softmax values for $\qb$ and $\vb$.
\end{lemma}
\begin{proof}
{Throughout, we assume $\eps\leq \frac{\Delta}{8}$ and $\Delta > 0$ where $\Delta$ is the minimum embedding distance, i.e.,  $ \Delta=(1-\max_{\ab\neq\bb\in\Vc}(\ab^\top\bb)^2)^{1/2}.$} 
{Before proceeding, we first note that, without losing generality, we can assume $L\geq W + 1$. The reason is that, if $L\leq W$, left or right end of the convolutional filter will never interact with features. Thus, we simply set them to zero, truncating the filter. }
Define sequence $\X^i \in \R^{L \times d}$ where $\x^i_{L - 1}=\qb$, $\x^i_i=\qb$ and $\x^i_j=\vb$ for all $j\neq i$. Let $\Z^i=\Fb\ast\X^i$. Let $\s^i=\sft{\X^i\W\qb}$ and $s_q=\s^i_{L - 1}$ and $s_v=(1-2s_q)/(L-2)$. Here $s_q$ and $s_v$ are the softmax values for $\qb$ and $\vb$ respectively. Additionally, observe that 
\[ 
s_v/s_q=\exp((\vb-\qb)^\top\W\qb).
\]
Finally, let $\Ic=[L]-\{L - 1,i\}$ and recalling value sequence $\Z^i$, note that 
\[ 
f(\X^i)=s_v\sum_{j\in\Ic}\z^i_j+s_q(\z^i_i+\z^i_L).
\]
By assumption, we also have that
\begin{align}
&\tn{\vb-f(\X^i)}\leq \eps\quad\text{for}\quad i<L-2,\quad \tn{\qb-f(\X^{L-2})}\leq \eps.
\end{align}
We will leverage these inequalities to prove the statement of the theorem. 
Let $\rho = \rho (\qb, \vb) =\qb^\top\vb$ be the correlation between $\qb,\vb$. Define $\vb^\perp=\frac{\qb-\rho\vb}{\tn{\qb-\rho\vb}}$. Observe that convolution output has the form $f(\X^i)=\alpha\vb+\beta\qb$ for some $\alpha=\alpha_i,\beta=\beta_i>0$. For $i<L-2$, we have that
\[ 
\eps\geq  \tn{\vb-f(\X^i)}\geq {|(\vb^\perp)^{\top}(\vb-f(\X^i))|\geq \beta|(\vb^\perp)^{\top}\qb|} \geq \beta\sqrt{1-\rho^2}.
\] 
Recalling that the minimum embedding distance is defined as $\Delta=\sqrt{1-\max_{\qb,\vb}\rho^2(\qb,\vb)}\leq 1$ and setting $\eps_0=\eps/\Delta$, this implies that 
\begin{align}
&\beta_i\leq \eps_0:=\eps/\Delta\quad \text{for}\quad i<L-2,\quad \alpha_{L-2}\leq \eps_0:=\eps/\Delta.
\end{align}
Additionally, writing $\eps\geq | \vb^\top(\vb-f(\X^i))|= |1-\alpha_i-\beta_i\vb^\top\qb|$ for $i<L-2$ and using $|\vb^\top\qb|\leq 1$, we can deduce
\begin{align}
&\alpha_i\geq 1-(1+1/\Delta)\eps\geq 1-2\eps_0\quad \text{for}\quad i<L-2,\quad \beta_{L-2}\geq 1-(1+1/\Delta)\eps\geq 1-2\eps_0\\
&\alpha_i\leq 1+(1+1/\Delta)\eps\leq 1+2\eps_0\quad \text{for}\quad i<L-2,\quad \beta_{L-2}\leq 1+(1+1/\Delta)\eps\leq 1+2\eps_0.
\end{align}
We note that when $L = W + 1$, the problem only has a subtle difference, which we discuss at the end. 

\textbf{Case 1: $L \geq W + 2$.} For $i=0$ and $i=L-2$, the coefficients $\alpha_i,\beta_i$ can be written in terms of convolution as 
\begin{align}
&\beta_0=2s_q F_0+s_v\sum_{i\neq 0} F_i\label{eq_b1}\\
&\beta_{L-2}=s_q (2F_0+F_{-1}+F_1)+s_v[2\sum_{i<0} F_i-(F_{-1}+F_{-W})].
\end{align} 
Let $\bar{F}_1=F_{-1}+F_1$. Observing $2\sum_{i<0} F_i-(F_{-1}+F_{-W})\leq 2(\sum_{i\neq 0} F_i)$, we can write 
\begin{align}\label{estimate1}
1-(1+1/\Delta)\eps\leq \beta_{L-2}\leq s_q \bar{F}_1+2\beta_0\leq s_q \bar{F}_1+2\eps/\Delta.
\end{align}
Combining these implies $s_q\bar{F}_1\geq 1-4\eps_0$. Also, we know the trivial bound $s_q\bar{F}_1\leq \beta_{L-2}\leq 1+2\eps_0$. Thus, we obtain 
\[ 
1+2\eps_0\geq s_q\bar{F}_1\geq 1-4\eps_0.
\] 
To proceed, we wish to prove that $s_v$ is small. From \eqref{eq_b1}, we have that $s_v\bar{F}_1\leq \eps_0$. Consequently, we have that
\[ 
\frac{s_v}{s_q}\leq \Gamma=\frac{\eps_0}{1-4\eps_0}.
\] 
Using $2s_q+(L-2)s_v=1$, we get
\[ 
1=2s_q+(L-2)s_v\leq (2+(L-2)\Gamma) s_q\implies s_q\geq \frac{1}{2+(L-2)\Gamma}=\frac{1-4\eps_0}{2+(L-10)\eps_0}.
\] 

Since $s_q\leq 1/2$ (due to query repeating twice), this also implies that 
\[ 
2\frac{(1+L\eps_0)(1+2\eps_0)}{1-4\eps_0}\geq 
2\frac{(1+2\eps_0)(1+(L/2-5)\eps_0)}{1-4\eps_0}\geq \bar{F}_1\geq 2(1-4\eps_0).
\] 
Using above, in essence, so far we have established that $|\bar{F}_1-2|\leq \order{L\eps_0}$ and $s_v/s_q\leq \order{\eps_0}$. Both statements hold whenever $\eps_0\leq 1/8$ (e.g.~so that $1/(1-4\eps_0)\leq 1+\order{\eps_0}$). The primary remaining item in the proof is establishing $|F_i|\leq\order{\eps_0}$ for all $i\neq -1$.

To prove this, we utilize the following observations: First, by keeping track of the contributions of the last two $\qb$ vectors on $\alpha_{L-2}$, we observe that 
\[ 
s_q\sum_{i=1}^WF_{i}\leq \alpha_{L-2}\leq \eps_0.
\]
This implies {$\sum_{i=1}^WF_{i}\leq \eps_0/s_q\leq \eps_0\frac{2+(L-10)\eps_0}{1-4\eps_0}=\order {\eps_0(1+L\eps_0)}$}. We similarly find $F_0\leq \eps_0/2s_q$ through \eqref{eq_b1}. Finally, since $F_1\leq \order{\eps_0(1+L\eps_0)}\leq \order{L\eps_0}$, we also find the critical bound 
\[ 
|F_{-1}-2|\leq \order{L\eps_0}.
\] 

Finally, we wish to bound $\sum_{i\leq -2}F_{i}$. To do so, we can bound the contribution of the first $\qb$ vector on $\beta_i$ as follows. For any $W\geq j\geq 2$, letting $i = L - 1 - j$, we have that
\[ 
{\eps_0\geq \beta_i\geq s_qF_{-j}} \implies F_{-j}\leq \eps_0\frac{2+(L-10)\eps_0}{1-4\eps_0}=\order{\eps_0(1+L\eps_0)}.
\] 
Aggregating these, we have found the advertised bounds: 
\begin{align}
&\tone{\Fb-\Db_{-1}}\leq \order{W\eps_0(1+L\eps_0)+L\eps_0}\leq \order{L\eps_0(1+W\eps_0)}\label{main bounds}\\ 
&\tone{\Fb_{\geq 0}}=\sum_{i=0}^{W}F_{i}\leq \order{\eps_0(1+L\eps_0)}\\
&o_q=s_v/s_q\leq \Gamma=\frac{\eps_0}{1-4\eps_0}=\order{\eps_0}
\end{align}
where $\vb$ is chosen to be the most similar token in terms of attention probabilities. Note that, the bound on left entries of $\Fb$ that retrieves the past values is tighter than the right entries.

\textbf{Case 2: $L = W + 1 $. } 
In this scenario, the main difference is we have the following estimates rather than \eqref{eq_b1}
\begin{align}
&\beta_0=s_q (2F_0+F_W+F_{-W})+s_v\sum_{i\neq 0,|i|< W} F_i\label{eq_b2}\\
&\beta_{L-2}=s_q (2F_0+\bar{F}_1)+s_v[2\sum_{i<0} F_i-(F_{-1}+F_{-W})].
\end{align}
So we can't immediately use the estimate provided right below \eqref{estimate1} because of the missing $F_{-W}s_v$ term. On the other hand, considering $\X^1$ and contribution of the first $\vb$ token on $\beta_1$, we have that $s_vF_{-W}\leq \beta_1\leq \eps_0$. As a result, we can instead use the fact that $\beta_0+\beta_1\leq \order{\eps_0}$ and the fact that 
\[ 
2s_qF_0+s_v(2\sum_{i<0} F_i-(F_{-1}+F_{-W}))\leq 2(\beta_0+\beta_1)
\]
so that we have again established $|1-s_q\bar{F}_1|\leq \order{\eps_0}$ and can proceed similarly.
\end{proof}
Now that we have established the fine-grained control of the filter and attention map with Lemma \ref{lemma filter ar}, we can conclude with length generalization. 
\begin{proof}[Proof of Theorem ~\ref{thm NAR main}]
Given a query $\qb$ and a sequence of length $L'$, let us define $s_q$ similarly (i.e.~attention probability that falls on the $\qb$ token) and study the attention output. Let $\qb$ appear at $i$ for the first time, $\vb$ be the token following $\qb$, and $\Ic=[L']-\{i,L' - 1\}$. Let $\ab=\sft{\X\W\qb}\in\R^{L'}$ be softmax scores with $a_i=a_{L' - 1}=s_q$. We write 
\[ 
f(\X)=\sum_{j\in \Ic}a_j \z_j+s_q(\z_i+\z_{L'-1}).
\] 
where $\z_j = \sum_{i = -W}^W F_i \x_{j - i}$
To proceed, let $R$ be a universal constant and {$\Xi=RL\eps_0(1+W\eps_0)$} so that $\tone{\Fb}\leq 2+\Xi$ from \eqref{main bounds} in Lemma~\ref{lemma filter ar}. Then we get $\tn{\z_j} \leq \tone{\F} \leq 2+\Xi$ for all $j \in [L']$. Secondly, due to right-clipped convolution we have $\tn{\z_{L' - 1}} \leq \tone{\sum_{i = 0}^W F_i} \leq \Xi$ and thanks to value retrieval at $i$'th position, we get
\begin{align}
    \tn{\z_i-2\vb}\leq |F_{-1} - 2|\tn{\vb} + |\sum_{j \neq -1}F_j| \leq \Xi
\end{align}
Next, observe that $a_j/s_q \leq s_v / s_q \leq \Gamma = \frac{\eps_0}{1 - 4\eps_0}$ for all $j\in\Ic$ and that $2s_q + \sum_{j \in \Ic} a_j = 1$, consequently, for some constant $R_0 >0$,
\[ 
\frac{1}{2}\geq s_q\geq \frac{1}{2+(L'-2)\Gamma}=\frac{1-4\eps_0}{2+(L'-10)\eps_0}\implies {|2s_q-1|\leq R_0 L'\eps_0}.
\] 
and 
\[ 
\sum_{j\in\Ic}a_j = 1 - 2 s_q \leq R_0L'\eps_0.
\]
Aggregating these, we find that 
\begin{align} 
\tn{f(\X)-\vb}
&\stackrel{(a)}\leq 
\tn{\sum_{j\in\Ic}a_j\z_j}+\tn{s_q(\z_i+\z_{L' - 1})-2s_q\vb}+|2s_q-1|\tn{\vb}
\\&\stackrel{(b)}\leq |\sum_{j\in\Ic}a_j|(2+\Xi)+\tn{s_q(\z_i+\z_{L' - 1})-2s_q\vb}+|2s_q-1|
\\
&\le R_0(2+\Xi)L'\eps_0+\Xi+R_0L'\eps_0\\
&\leq 3R_0L'\eps_0+\Xi+R_0 \Xi L'\eps_0\\
&\leq 3\eps_0\big(R_0 L' + RL(1+W\eps_0)(1+R_0L'\eps_0)\big)
\end{align}
where (a) follows triangle inequality and (b) follows Cauchy-Schwarz inequality. Let $c_0,c_1$ be absolute constants to be determined. Assuming $W\eps_0\leq \order{1}$ (i.e.~bounded by constant), we have that 
\[ 
\tn{f(\X)-\vb}\leq c_0\eps_0(L'+L+LL'\eps_0)
\] 
where $c_0 \geq 3 \max\{R_0, R(1 + W\eps_0), R_0R(1+W\eps_0)\}$. Assuming the stronger bound $L\eps_0\leq \order{1}$ and $c_1 \geq c_0(1 + L / L' + L\eps_0) $, we have that 
\[ 
\tn{f(\X)-\vb}\leq c_1\eps_0L'
\] 
This concludes the advertised proof.
\end{proof}

\subsection{Proving Length Generalization for N-gram AR (Proposition~\ref{prop len gen nar})}\label{sec nar len gen}
\shaw{In this section we use $\F_v = \F$ for the filter applied on value token and $\F_q = \F_k = \bar{\F}$ for filters on query and keys}.
\begin{assumption} \label{assumption bound} {Recall that 
$\Vc$ is the vocabulary from which the token embeddings are drawn}.
\shaw{We have the following two assumptions to make the output $f(\X)$ more tractable}:
\begin{enumerate}[label={\textbf{~\alph*)}}, ref={\theassumption.\alph*}]
\item \label{assump-bound-A}%
    The filter weights are bounded and obey $\tone{\F} \leq 1$. Besides, assuming that $\Delta = 1 - \max_{\ab, \bb \in \Vc, \ab \neq \bb}\bb^{\top}\ab > 0$
\item \label{assump-bound-B}%
    Any subset of $2N$ tokens within the vocabulary $\Vc$ is linearly independent.
\end{enumerate}
{Note that Assumption~\ref{assump-bound-B} is essentially a restricted isometry property condition on the embedding matrix induced by the vocabulary $\Vc$. Specifically, if embeddings are randomly chosen, as soon as the embedding dimension obeys $d\gtrsim \order{N\log \frac{|\Vc|}{N}}$, this assumption will hold with high probability \cite{candes2008restricted,2006-candes-tao}}. \shaw{In the following analysis, we will leverage either one of the assumptions to establish the length generalization result}.
\end{assumption}

\begin{lemma}\label{lemma GS}
    Suppose Assumption~\ref{assump-bound-B} holds. Let $\Bc$ be any subset of $2N$ tokens within $\Vc$ and $\mathcal{U} := \{\ub_j | j \in [|\mathcal{U}|]\}$ be the orthonormal tokens obtained after applying the Gram-Schmidt process on $\Bc$ where $\bb_j = \sum_{l = 0}^j \beta_{j,l}\ub_l$. Then we have $0 < \delta = \min_{j \in [|\Bc|]} |\beta_{j, j}| \leq 1$. 
\end{lemma}
\begin{proof}
    First note that $\beta_{j,l} = \bb_j^{\top} \ub_l \leq 1 $ for any $j,l \in [|\Bc|]$ and $\beta_{0, 0} = 1$. Then $\delta = \min_{j \in [|\Bc|]} |\beta_{j, j}| \leq 1$. Moreover, we can prove $\delta > 0$ by contradiction. Assuming there exists $j \geq 1$ such that $\beta_{j,j} = 0$. This indicates that $\bb_j$ can represented as a linear combination of the previously orthogonalized vectors $\{\ub_0, \dots, \ub_{j - 1}\}$. In other words, $\bb_j$ lies entirely in the span of these previous vectors. This contradicts the fact that tokens in $\Bc$ are linearly independent. As a result we have $\delta > 0$.
\end{proof}
\begin{proposition}\label{prop len gen nar}Let $\bar{\F} \in \R^{N}$ be a 1-D causal convolutional filter and {$\F \in \R_{+}^{2W+1}$ be a 1-D convolutional filter from time $t = -W$ to $t = W$} \shaw{where $W \leq L - N$}.
    Suppose that token embeddings have unit norm.
    Consider the same CAT Layer $f(\X)=(\X_v\W_v)^{\top}\sft{\X_k\W_k\W_q^{\top}\qb}$ defined in Theorem~\ref{thm NAR main} where $\qb$ is the final token of $\X_q$ and $\X_q=\norm{\X\ast\Fb_q} \in \R^{L \times d}$ (same for $\X_k,\X_v$). We set $\Fb_q = \F_k = \bar{\F}$, $\W = \W_k\W_q^{\top}, \text{ and }\W_v = 2 \Ib_d$
    Consider any $f = (\W, \F) $ that can solve the N-AR problem up to $\eps$-accuracy on all sequences of length \shaw{$L\geq \order{N}$}. That is, for all $(\X,\y)$ where N-gram $\Z$ occurs within $\X$ exactly twice and $\y$ being the associated value token that follows the first occurrence of $\Z$, we have $\tn{\y-f(\X)}\leq\eps$. 
    Let $\Bc$ is any subset of $2N$ tokens within vocabulary $\Vc$ and $\beta_{j, j}$ be the corresponding projection coefficients defined in Lemma~\ref{lemma GS}. Assume either Assumption~\ref{assump-bound-A} or~\ref{assump-bound-B} holds and define 
    \begin{align*}
        \eps_0 = \begin{cases}
        \eps / \Delta, \quad \Delta = 1 - \max_{\ab, \bb \in \Vc, \ab \neq \bb}\bb^{\top}\ab & \text{under Assumption~\ref{assump-bound-A}} \\ 
        \eps \frac{e^{2N / \delta}}{\delta}, \quad \delta = \min_{j \in [|\Bc|]} |\beta_{j,j}| & \text{under Assumption~\ref{assump-bound-B}} 
    \end{cases}
    \end{align*}
    For almost all choices of $\bar{\Fb}$, there are absolute constants $R_0,R>0$ such that, {if $\eps_0 \leq R_0/L$}, we have that
    \begin{itemize}
        \item $\tone{\Fb - \Db_{-1}} \leq L \eps_0$
        \item Let $\s_\st\in\R^{L'}$ be a vector with entries equal to $1/2$ at the positions of query $\qb$ in $\X_q$ and $0$ otherwise. For all inputs $\X$ of arbitrary length $L'$, attention map obeys $\tone{\sft{\X_k\W\qb}-\s_\st}\leq L'\eps_0$.
        \item For all N-AR sequences $\X$ of arbitrary length $L'$, we have that $\tn{\y-f(\X)}\leq RL'\eps_0$.
    \end{itemize}
\end{proposition}

\begin{lemma}\label{lemma filter nar}
    Consider the same setting in Prop.~\ref{prop len gen nar}, for any $f=(\W,\F)$ that can solve the N-AR problem defined in Def.~\ref{NAR prob def} up to $\eps$-accuracy on all sequences of length \shaw{$L \geq \order{N}$}. There are absolute constants $R_0>0$ such that, if {$\eps_0 \leq R_0/N$}, we have that
    \begin{align}
    & \tone{\F - \Db_{-1}} \leq \order{N\eps_0(1 + L \eps_0) + N \eps_0} \leq \order{L\eps_0(1 + N \eps_0)}
    \\ & \tone{{\F}_{\geq 0}} = \sum_{i = 0}^W F_i \leq \order{N\eps_0(1 + L \eps_0)}
    \end{align}
    where we use $\order{\cdot}$ notation to denote an upper bound up to a constant i.e. for some absolute $r>0$, $\order{x}\leq r\cdot x$. 
    Moreover, we consider N-gram $\Z_q \in \R^{N \times d} $ that ends with a token $\qb'$, which can be any token from the vocabulary $\Bc_N$. Let $\qb$ be the final token of $\norm{\Z_q \ast \bar{\F}}$ and $\vb$ be the top query not equal to $\qb$ that maximizes the similarity $\vb^{\top}\W\qb$. i.e.~$\vb=\arg\max_{\x\in\Bc_N,\x\neq\qb}\x^{\top}\W\qb$, we have
    \begin{align}
        o_q = \frac{s_v}{s_q} \leq {\Gamma = \frac{\order{\eps_0}}{1 - \order{N \eps_0}} \leq \order{N\eps_0}}
    \end{align}
    where $s_q$ and $s_v$ are the softmax values for $\qb$ and $\vb$.
\end{lemma}
\begin{proof}
Following the proof of Thoerem~\ref{thm NAR main}, \shaw{suppose that we have an $\bar{\F}$ that results in unique signatures.}
\shaw{We argue that the length generalization fails when $W > L-N$}, which is explained at the end. 
Throughout, we assume that $W = L - N$. When $W < L - N$, it is equivalent to the setting where $W = L - N$ and $F_{j} = 0$ for $W + 1\leq |j| \leq L - N $. Denote the corresponding N-gram that results in $\qb$ and $\vb$ after convolving with $\bar{\F}$ be {$\Z_q = [\qb_0,\qb_1,\dots,\qb_{N - 1}] \in \R^{N \times d}$} and $\Z_v = [\vb_0,\vb_1,\dots,\vb_{N - 1}] \in \R^{N \times d}$ respectively, i.e., $\qb = \norm{\Z_q \ast \bar{\F}} \text{ and }\vb = \norm{\Z_v \ast \bar{\F}}$. $\Z_q$ and $\Z_v$ are unique due to the assumption on $\bar{\F}$. For brevity, let $\qb' = \qb_{N - 1}, \vb' = \vb_{N - 1}$ and $\Z_{v'}^k = [\vb',\vb',\dots,\vb'] \in \R^{k \times d}$, $\Z_q' = [\qb_1,\dots,\qb_{N - 1}] \in \R^{(N - 1) \times d}$, where $\qb_0$ is removed from $\Z_q$.
\begin{align}
    \X^{i, k} = \Big[\Z_{v'}^{N - 1 + i}  \quad \Z_v \quad \Z_q \quad \Z_{v'}^{N - 1 + k} \quad \Z_q' \quad \qb_0 \quad \Z_{v'}^{N - 1}\quad \Z_{v'}^{n_{i,k}} \quad \Z_v \quad \Z_q \Big] \in \R^{L \times d}
    \\ \Xb^{i, k} = \Big[\Z_{v'}^{N - 1 + i}  \quad \Z_v \quad \Z_q \quad \qb_0 \quad \Z_{v'}^{N - 1 + k} \quad \Z_q' \quad \Z_{v'}^{N - 1} \quad \Z_{v'}^{n_{i,k}}\quad \Z_v \quad \Z_q \Big] \in \R^{L \times d}
\end{align}
where $n_{i, k} = L - 8N + 3 - i - k \geq 0$, and this naturally introduces a lower bound for $L$, i.e., $L \geq \order{N}$ and upper bounds for both $i$ and $k$. Note that $\X^{i, k} \text{ and } \Xb^{i, k}$ have different labels. By assumption, we have that
\begin{align}
&\tn{\vb'-f(\X^{i, k})}\leq \eps ,\quad \tn{\qb_0 -f(\Xb^{i, k})}\leq \eps.
\end{align}
Let $\s^{i, k} = \sft{\X^{i, k}\W\qb}, \bar{\s}^{i, k} = \sft{\Xb^{i, k}\W\qb}$. Define the probability of selecting the $j$-th entry of $\X^{i, k}$ and $\Xb^{i, k}$ as $s_j^{i, k} \text{ and } \bar{s}_j^{i, k}$ and selecting the token $\qb \text{ and } \vb$ as $s_q, s_v$. Here we omit $i, k$ for $s_q $ and $s_v$ since it's invariant to the values of $i, k$.
Additionally, observe that 
\[ 
s_v/s_q=\exp((\vb-\qb)^\top\W\qb) \text{ and } ( L - 2)s_v + 2s_q \geq 1
\]
where the inequality comes from the fact that $\vb = \arg\max_{\x\in\Bc_N,\x\neq\qb}\x^{\top}\W\qb$.
We will leverage these inequalities to prove the statement of the theorem. Define the vocabulary set $\Bc = \{\vb_0, \vb_1, \dots,\vb_{N-1},\qb_0,\qb_1,\dots,\qb_{N-1}\}$ which includes all tokens in $\X^{i, k} \text{ and } \Xb^{i, k}$. Note that \shaw{the vocabulary $\Bc$ is a subset of tokens chosen from $\Vc$, i.e., $\Bc \subseteq \Vc$} and the vocabulary size $|\Bc|$ is at most $2N$, i.e., $|\Bc|:= K \leq 2N$. 
Observe that convolution output has the form 
{$f(\X^{i, k}) = \sum_{j \in [|\Bc|]} m_j^{i, k}\bb_j$ and $f(\Xb^{i, k}) = \sum_{j \in [|\Bc|]} \bar{m}_j^{i, k}\bb_j$} where $\{m_j^{i, k}, \bar{m}_j^{i, k}\}_{j \in [|\Bc|]}$ are non-negative coefficients due to the assumption that entries in $\bar{\F}$ and softmax probabilities $\s^{i, k} \text{ and } \bar{\s}^{i, k}$ are non-negative. In particular, we are interested in $m_{q}^{i, k}, \bar{m}_{q}^{i, k}$ and $m_{v}^{i, k}, \bar{m}_{v}^{i, k}$, which correspond to the coefficients of token $\qb_0$ and $\vb'$. 
To proceed, we leverage Assumption~\ref{assumption bound} to bound the coefficients:

\textbf{When Assumption~\ref{assump-bound-A} holds. } By expanding the coefficients, we get 
\begin{align}
    \sum_{j \in [|\Bc|]}m_{j}^{i,k} = \sum_{j \in [|\Bc|]} \sum_{t \in [L]} {F_v}{t - j}s_{t}^{i, k} \leq \tone{\F} \sum_{t \in [L]}s_{t}^{i,k} \leq 1
\end{align}
Combining this with the fact that
\begin{align} \label{eq: landscape v}
    \eps 
    \geq \tn{\vb' - f(\X^{i, k})}\geq |\vb'^{\top}(\vb' - f(\X^{i, k}))| \geq |\sum_{j \in [|\Bc|], \bb_j \neq \vb'}\vb'^{\top}\bb_j m_j^{i,k} + m_v^{i,k} - 1|
\end{align}
    , we have
    \begin{align}
        \eps  \geq \sum_{j \in [|\Bc|], \bb_j \neq \vb'}(1 - \vb'^{\top}\bb_j)m_j^{i,k}  \geq (1 - \vb'^{\top}\bb_j)m_{j}^{i,k} \quad \text{for any} \quad j \in \{ j \: | \: j \in [|\Bc|], \bb_j \neq \vb'\}
        \\ \label{eq: normal bound m}
        \to m_{j}^{i,k}  \leq \eps / \Delta := \eps_0 \quad \text{for any} \quad j \in \{ j \: | \: j \in [|\Bc|], \bb_j \neq \vb'\}
    \end{align}
    where $\Delta = 1 - \max_{\ab, \bb \in \Bc, \ab \neq \bb}\bb^{\top}\ab > 0$. In terms of $m_{v}^{i,k}$, we apply Triangle Inequality on \eqref{eq: landscape v} and \eqref{eq: normal bound m}:
    \begin{align}
        |1 - m_{v}^{i,k}| &\leq |\sum_{j \in [|\Bc|], \bb_j \neq \vb'}\vb'^{\top}\bb_j m_j^{i,k} + m_v^{i,k} - 1| + |\sum_{j \in [|\Bc|], \bb_j \neq \vb'}\vb'^{\top}\bb_j m_j^{i,k}| \\ 
        &\leq \eps + 2N \eps_0 \leq (2N+1) \eps_0
    \end{align}
Similarly for $\Xb^{i,k}$ we have
\begin{align}
    \bar{m}_j^{i,k} \leq \eps_0 \quad \text{for any} \quad j \in \{ j \: | \: j \in [|\Bc|], \bb_j \neq \qb_0\}, \quad 
    |1 - \bar{m}_q^{i,k}| \leq \order{N\eps_0}
\end{align}
\textbf{When Assumption~\ref{assump-bound-B} holds. }\shaw{Based on the linear independence property, we can apply the Gram–Schmidt process to transform the tokens in $\Bc$ to orthonormal tokens $\mathcal{U} = \{\ub_j | j \in [|\mathcal{U}|]\}$ where $\bb_j = \sum_{l = 0}^j \beta_{j,l}\ub_l$} where $\beta_{j,l} = \bb_j^{\top} \ub_l $. Since the order of tokens in $\mathcal{U}$ does not matter, we can set $\ub_0 = \vb'$. Then for any $j \geq 1$, $\ub_{j}$ is orthogonal to $\vb'$ and $\bb_i$ for all $i < j$. Consider the case of $\X^{i, k}$ whose label is $\vb'$, utilizing the orthogonality we get
\begin{align} \label{eq st case}
    \eps \geq \tn{\vb' - f(\X^{i, k})}\geq |\ub_j^{\top}(\vb' - f(\X^{i, k}))| \geq |\sum_{l = j}^{|\Bc| - 1} m_{l}^{i, k} \ub_j^{\top}\bb_l|
\end{align}
Using backward induction, we can then bound $m_j^{i, k}$ for $1 \leq j \leq K - 1$. First consider $j = |\Bc| - 1 = K - 1$. Then we have:
\begin{align}
    \eps \geq |m_{K-1} \ub_{K-1}^{\top}\bb_{K-1}| = |m_{K-1}\beta_{K-1, K-1}| \geq |m_{K-1}\delta|
\end{align}
\shaw{where $\delta = \min_{j \in [|\Bc|]} |\beta_{j, j}| = \min_{j \in [|\Bc|]} |\bb_j^{\top}\ub_j|$}. Following Lemma~\ref{lemma GS} we have $0 < \delta \leq 1$. As a result we get $m_{K-1} \leq \eps / \delta$. 
Next we prove that if $j \geq 1$ and $m_l^{i,k} \leq \eps \frac{(1 + 1/\delta)^{K - l -1}}{\delta}$ for $j < l \leq K - 1$, $m_{j} \leq \eps \frac{(1 + 1/\delta)^{K - j - 1}}{\delta}$. 
When $1 \leq j \leq K - 2$, from equation~\eqref{eq st case} we can derive
\begin{align}
    \eps \geq |\sum_{l = j}^{K-1} m_l^{i,k} \ub_j^{\top}\bb_l | 
    = |\sum_{l = j + 1}^{K-1} m_l^{i,k} \ub_j^{\top}\bb_l + m_j\ub_j^{\top}\bb_j|
\end{align}
For the first term we have
\[
    |\sum_{i = j + 1}^{K-1} m_i \ub_j^{\top}\bb_i| \leq \sum_{i = j + 1}^{K-1} m_i \leq \sum_{i = 0}^{K - j - 2} \eps \frac{(1 + 1/\delta)^{i}}{\delta} = \eps \big((1 + 1/\delta)^{K - j -1} - 1\big)
\]
Using Triangle Inequality we get
\begin{align}
    |m_j\ub_j^{\top}\bb_j| \leq \eps (1 + 1/\delta)^{K - j -1} \to m_{j} \leq \eps \frac{(1 + 1/\delta)^{K - j -1}}{\delta} \leq \eps \frac{e^{\frac{K - j - 1}{\delta}}}{\delta} \leq \eps \frac{e^{2N / \delta}}{\delta}
\end{align}
We can hereby bound $m_j^{i,k}$ for any $j \in \{ j \: | \: \bb_j \in \Bc, \bb_j \neq \vb'\}$. Let $\eps_1 := \eps \frac{e^{2N / \delta}}{\delta}$, we have
\begin{align}
    & m_{j}^{i, k} \leq {\eps_1} \quad \text{for any} \quad j \in [|\Bc|], \bb_j \neq \vb'
\end{align}
Additionally, writing $\eps\geq | \vb'^\top(\vb'-f(\X^{i, k}))|= |1- m_v^{i, k}- \vb'^\top \sum_{j \in [|\Bc|], \bb_j \neq \vb'}m_j^{i, k} \bb_j|$ and using $|\vb'^\top \bb_j|\leq 1$ for any $\bb_j \in \Bc$, we can deduce 
\begin{align}
    |1 - m_{v}^{i,k}| 
    & \leq \eps + \sum_{i = 1}^{K - 1}m_i \\ 
    & \leq \eps + \sum_{i = 0}^{K - 2}\eps \frac{(1 + 1/\delta)^i}{\delta} \\ 
    & \leq \eps (1 + 1/\delta)^{K - 1} \leq \eps_1
\end{align}
Similarly for $\Xb^{i, k}$, we have
\begin{align}
\bar{m}_{j}^{i, k} \leq {\eps_1} \quad \text{for any} \quad   j \in [|\Bc|], \bb_j \neq \qb_0, \quad |1 - \bar{m}_{q}^{i, k}| \leq \eps_1
\end{align}
To summarize, using Assumption~\ref{assumption bound}, we can have an upper bound on $\{m_j^{i, k}, \bar{m}_j^{i, k}\}_{j \in [|\Bc|]}$:
    \begin{align}
        m_{j}^{i,k} \leq \eps_0 &\quad \text{for any} \quad j \in \{ j \: | \: j \in [|\Bc|], \bb_j \neq \vb'\},  \quad |1 - m_{v}^{i,k}| \leq \order{N \eps_0} \\ 
        \bar{m}_{j}^{i,k} \leq \eps_0  &\quad \text{for any} \quad j \in \{ j \: | \: j \in [|\Bc|], \bb_j \neq \qb_0 \},  \quad |1 - \bar{m}_{q}^{i,k}| \leq \order{N \eps_0}
    \end{align}
    where 
    \begin{align}\eps_0 := \begin{cases}
        \eps / \Delta, \quad \Delta = 1 - \max_{\ab, \bb \in \Bc, \ab \neq \bb}\bb^{\top}\ab > 0 & \text{under Assumption~\ref{assump-bound-A}} \\ 
        \eps \frac{e^{2N / \delta}}{\delta}, \quad \delta = \min_{j \in [|\Bc|]} |\beta_{j,j}| & \text{under Assumption~\ref{assump-bound-B}} 
    \end{cases}
    \end{align}

We proceed by comparing $m_{q}^{i, k} \text{ and }\bar{m}_q^{i, k}$:
\begin{align}
    m_{q}^{i, k}  &= 2s_q \big( 
    F_{-L + 4N + i - 2} + F_{-2N + 1 - k} + 2F_{N - 1} + F_{L - 5N - i - k} +  F_{L - 2N - i}\big) 
    \\ &+ 2\sum_{j \in [L] - \{3N -2 + i, L - 1\}} s_j^{i, k}(F_{-L + N + j} + F_{-5N - i - k + j + 3} + F_{-2N - i + j + 1}) \label{eq sv}
    \\ \bar{m}_q^{i, k} & = 2s_q \big( 
    F_{-L + 4N + i - 2} + F_{-1} + 2F_{N - 1} + F_{L - 3N - i} +  F_{L - 2N - i}\big) \label{eq mb sq }
    \\ &+ 2\sum_{j \in [L] - \{3N -2 + i, L - 1\}}\bar{s}_j^{i, k}(F_{-L + N + j} + F_{-3N - i + j + 1} + F_{-2N - i + j + 1})
\end{align}
Observing that 
\begin{itemize}
    \item $s_{j}^{i, k} = \bar{s}_{j}^{i, k}$ for $j \in [3N-1+i]$ and $ 6N - 3 + i + k \leq j \leq L - 1$
    \item $s_{j + 1}^{i, k}  = \bar{s}_{j}^{i, k}$ for $4N - 1 + i+ k \leq j \leq 5N - 3 + i + k$
    \item $s_{j + (2N - 1)}^{i, k} = \bar{s}_{j}^{i, k}$ for $5N - 2 + i + k_1 \leq j \leq 6 N - 4 + i + k$
    \item $s_{j + 2N - 2 + k_1 + i_1 - i_2}^{i_1, k_1}  = \bar{s}_{j}^{i_2, k_2} $ for $3N-1 + i_2 \leq j \leq 4N - 2 + i_2$
\end{itemize}
Utilizing these observations, we have
\begin{align*}
    \sum_{j \in [3N-1+i]}2  \bar{s}_j^{i, k}(F_{-L + N + j} + F_{-3N - i + j + 1} + F_{-2N - i + j + 1})
    & \leq m^{i, k}_q + m_q^{i + N, k}
    \\ \sum_{j \in [6N - 3 + i + k, L - 1]}2 \bar{s}_j^{i, k}(F_{-L + N + j} + F_{-3N - i + j + 1} + F_{-2N - i + j + 1})
    &\leq m_q^{i, k} + m_q^{i, k - N}
    \\ \sum_{4N - 1 + i + k \leq j \leq 5N - 3 + i + k} 2\bar{s}_j^{i, k}(F_{-L + N + j} + F_{-3N - i + j + 1} + F_{-2N - i + j + 1})
    & \leq m_q^{i + 1, k} + m_q^{i - N + 1, k} + m_q^{i, k + 1}
    \\ \sum_{5N - 2 + i + k \leq j \leq 6N - 4 + i + k} 2\bar{s}_j^{i, k}(F_{-L + N + j} + F_{-3N - i + j + 1} + F_{-2N - i + j + 1})
    & \leq m_q^{i + (N - 1), k} + m_q^{i + (2N - 1), k} + m_q^{i, k + (2N - 1)}
    \\ \sum_{3N - 1 + i \leq j \leq 4N - 2 + i} 2\bar{s}_j^{i, k}(F_{-L + N + j} + F_{-3N - i + j + 1} + F_{-2N - i + j + 1}) 
    & \leq m_q^{i - (2N-2), k} + \sum_{3N - 1 + i \leq j \leq 4N - 2 + i} \bar{s}_j^{i, k}(F_{-3N - i + j + 1} + F_{-2N - i + j + 1}) 
    \\ & \stackrel{(a)}\leq m_q^{i - (2N-2), k}  + s_v \sum_{l \in [2N]}F_l
    \\ & \stackrel{(b)}\leq m_q^{i - (2N-2), k}  + {\sum_{j \in \{j \: | \: \bb_j = \xb_l^{i, k}, l \in [2N]\}}\bar{m}_j^{1, k}} 
\end{align*}
where (a) comes from $\vb = \arg\max_{\x \in \Bc_N, \x \neq \qb} \x^{\top}\W\qb$ and (b) comes from the attention from the first $\vb$ to {itself and its previous $2N - 1$ terms}. Let $i = 2N - 2, k = N$, combining the inequalities above, we get
\begin{align}
    1 - \order{N\eps_0} \leq \bar{m}_q^{2N-2, N}
    &\leq 2s_q (F_{-1} + F_{L - 5N + 2}) + m^{2N-2, N}_q + m_q^{3N - 2, N} + ... +m_q^{0, N}  + {\sum_{j \in \{j \: | \: \bb_j = \xb_l^{i, k}, l \in [2N]\}}\bar{m}_j^{1, N}} 
    \\ &\leq 2s_q (F_{-1} + F_{L - 5N + 2}) + \order{N\eps_0}
\end{align}
Combining these we get $s_q (F_{-1} + F_{L - 5N + 2}) \geq 1/2 - \order{N \eps_0}$. Moreover, from \eqref{eq mb sq }, we have $s_q (F_{-1} + F_{L - 5N + 2}) \leq \bar{m}_{q}^{i, k} / 2 \leq 1/2 + {\order{N\eps_0}}$, which results in
\begin{align}
    | s_q (F_{-1} + F_{L - 5N + 2}) - 1/2 |\leq  \order{N\eps_0}
\end{align}
Based on this, we wish to show that $s_v$ is small. Substituting $l \in \{4N - 4, L - N\}$ into \eqref{eq sv}, we have $s_v(F_{-1} + F_{L - 5N + 2}) \leq m_{q}^{2N - 2, N} \leq \eps_0$, which implies that 
\begin{align}
    \frac{s_v}{s_q} \leq {\Gamma = \frac{\order{\eps_0}}{1/2 - \order{N \eps_0}}}
\end{align}
Using $2s_q + (L - 2)s_v \geq 1$, we have
\begin{align}
    1 \leq 2s_q + (L - 2)s_v \leq (2 + (L - 2)\Gamma)s_q \implies s_q \geq \frac{1}{2 + (L - 2)\Gamma} \geq \frac{1 - \order{N \eps_0}}{2 + (L - 2N - 2)\order{\eps_0}}
\end{align}
Combining this with $2s_q \leq 1$, we get
\begin{align}
    {\frac{(1 + L \order{\eps_0})(1 + \order{N\eps_0})}{1 - \order{N \eps_0}}} 
    &\geq \frac{(1 +  \order{N\eps_0})({1 + (L / 2 - N - 1)\order{\eps_0}})}{1 - \order{N \eps_0}} 
    \\ &\geq F_{-1} + F_{L - 5N + 2} \geq 1 - \order{N \eps_0}
\end{align}
At this stage, we have already proved that when ${N\eps_0 \leq \order{1}}$, $|F_{-1} + F_{L - 5N + 2} - 1| \leq \order{L \eps_0}$ and $s_v / s_q \leq \order{\eps_0}$. 
The primary remaining proof is to prove $|F_l| < \order{\eps_0}$ for all $l \neq -1$. Since $\bar{m}_j^{i, k} \leq \eps_0$ for $\bb_j \neq \qb_0$, by tracking the contribution of the last $\qb$ on $\bar{m}_j^{i, k}$, we have
\begin{align}
    s_q \sum_{0 \leq l \leq L - N, l \notin \{N - 1, L - 2N , L - 3N\}} F_l \leq \sum_{ j \in \{ j \: | \:\bb_j \in \Bc, \bb_j \neq \qb_0 \} } \bar{m}_j^{0, 0} \leq {\order{N\eps_0}} 
    \\ s_q \sum_{l \in \{L - 2N , L - 3N\}} F_l \leq \sum_{j \in \{j \: | \: \bb_j \in \{\qb', \vb'\}\} } \bar{m}_j^{1, 0} \leq {\eps_0}
    \\ s_q F_{N - 1} \leq m_q^{0, 0} \leq \eps_0
\end{align}
Combining these three inequalities and $W = L - N$, we get $\sum_{l = 0}^{W} F_l \leq \order{N\eps_0} / s_q \leq \order{N \eps_0}\frac{2 + (L - 2N - 2)\order{\eps_0}}{1 - \order{N \eps_0}} \leq \order{N\eps_0(1 + L\eps_0)}$. Additionally, since $F_{L - 5N + 2} \leq \order{N\eps_0(1 + L\eps_0)} \leq \order{NL\eps_0}$, we get 
\begin{align}
    |F_{-1} - 1 |< {\order{NL \eps_0}}
\end{align}
Finally, we want to bound $\sum_{-L + N \leq j \leq -2} F_{j}$. By tracking the contribution of the first $\qb$ to $\bar{m}_{j}^{0, 0}$ for any $j \in \{j \: | \: \bb_j \in \Bc, \bb_j \neq \qb_0\}$, we have $s_q F_{l} \leq \bar{m}_{j}^{0, 0} \text{ where }\bb_j = \x_{-l + 3N - 2}^{i, k}$ for any $-L + 3N - 1 \leq l \leq -2$. Summing over $l$ we get,
\begin{align} \label{eq:filter neg 1}
    {s_q \sum_{-L + 3N - 1 \leq l \leq -2, l \neq W - 4N + 3}F_{l} \leq \sum_{j \in \{j \: | \: \bb_j \in \Bc, \bb_j \neq \qb_0\} }\bar{m}_j^{0, 0} \leq {\order{N\eps_0}}}
\end{align}
Note that $F_{-L + 4N - 2}$ touches $\qb_0$ for $\Xb^{0, 0}$ so we need to handle it separately. We can consider $\bar{m}_j^{1, 0}$ instead and get
\begin{align} \label{eq:filter neg 2}
    s_q F_{-L + 4N - 2} \leq \bar{m}_{L - N + 1}^{1, 0} \leq {{\eps_0}}
\end{align}
For $\sum_{-L + N \leq j \leq -L + 3N - 2} F_{j}$, we consider the following example:
\begin{align}
\hat{\X}^{i, k} = \Big[\Z_q \quad \vb_0 \quad \Z_{v'}^{N - 1 + i}  \quad \Z_v \quad \Z_{v'}^{N - 1 + k} \quad \Z_q' \quad \Z_{v'}^{N - 1} \quad \Z_{v'}^{n_{i,k}}\quad \Z_q \Big] \in \R^{L \times d}
\end{align}
Similarly we define $\hat{\s}^{i, k} = \sft{\hat{\X}^{i, k}\W\qb}$ and the probability that selects $\qb$ and $\vb$ as $\hat{s}_q$ and $\hat{s}_v$. Note that $\hat{s}_v / \hat{s}_q = s_v / s_q \leq {\order{\eps_0}}$ and $(L - 2)\hat{s}_v + \hat{s_q} \geq 1$. We can then obtain the exact lower bound as $s_q$ for $\hat{s}_q$, i.e., $\hat{s}_q \geq \frac{1 - \order{N\eps_0}}{2 + (L - 2N - 2)\order{\eps_0}}$. Note that the output of $f(\hat{\X}^{i,k})$ can also be written as $f(\hat{\X}^{i,k}) = \sum_{\bb_j \in \Bc} \hat{m}_j^{i, k}\bb_j$ where $\{\hat{m}_j^{i,k}\}_{\bb_j \in \Bc}$ is the corresponding coefficients. Moving forward, by tracking the attendance of the first $\qb$ to the last $2N - 1$ terms, we have $\hat{s}_q F_{l} \leq \hat{m}_j^{0, 0}$ where $\bb_j = \hat{\x}^{i, k}_{-l + i + N - 1}$ for any $-L + N \leq j \leq -L + 3N - 2$. \shaw{Repeating the same argument based on Assumption~\ref{assumption bound}}, we get $\hat{m}_j^{0, 0} \leq {\eps_0}$ for any $j \in \{j \: | \: \bb_j \in \Bc, \bb_j \neq \vb_0\}$, which leads to
\begin{align} \label{eq:filter neg 3}
    {\hat{s}_q \sum_{-L + N \leq l \leq -L + 3N - 2}F_l \leq \hat{s}_q \sum_{j \in \{j \: | \: \bb_j \in \Bc, \bb_j \neq \vb_0 \}} \hat{m}_j^{0, 0} \leq \order{N\eps_0}}
\end{align}
Combining \eqref{eq:filter neg 1}, \eqref{eq:filter neg 2} and \eqref{eq:filter neg 3}, we have
\begin{align}
    \sum_{-L + N \leq l \leq -2}F_l \leq \order{N\eps_0}(1/ \hat{s}_q + 1 /s_q) \leq \order{N\eps_0}\frac{2 + (L - 2N - 2)\order{\eps_0}}{1 - \order{N\eps_0}} \leq \order{N\eps_0(1 + L \eps_0)}
\end{align}
In summary, we get
\begin{align}
    & \tone{\F - \Db_{-1}} \leq \order{N\eps_0(1 + L \eps_0) + N \eps_0} \leq \order{L\eps_0(1 + N \eps_0)}
    \\ & \tone{\F_{\geq 0}} = \sum_{i = 0}^W F_i \leq \order{N\eps_0(1 + L \eps_0)}
    \\ & o_q = \frac{s_v}{s_q} \leq {\Gamma = \frac{\order{\eps_0}}{1 - \order{N \eps_0}}} \leq \order{N\eps_0}
\end{align}
where $\vb$ is chosen to be the most similar tokens to $\qb$ in terms of attention probabilities. \shaw{Now we discuss the scenario where $W > L - N$. First note that the output $f(\X)$ can be written as $f(\X) = 2 \sum_{l \in [L]}s_l \sum_{j = -W}^W F_j \x_{l - j}$ where $s_l $ is the softmax value at $l$-th position. We are interested in $s_q$ in particular, which is proven to converge to $1/2$ when $\eps$ diminishes. Also, note that the smallest possible index of $\qb$ is $N - 1$ since it's the last token of an N-gram. Then, when $W > L - N$, the left end of the convolutional filter never interacts with $s_q$ since the index of $\x_{i - j}$ is out of bound, i.e., $i - j = N - 1 + W > L - 1$ } 
\end{proof}
Using the results from Lemma~\ref{lemma filter nar}, we can establish the length generalization on N-AR task.
\begin{proof}[Proof of Proposition~\ref{prop len gen nar}]Given a sequence $\X$ of length $L'$, let $\qb$ be the last token of $\X_q = \X_{k} = \norm{\X \ast \bar{\F}}$ and we define $s_q$ as the attention probability that selects $\qb$. Assume the first occurrence of $\qb$ in $\X_k$ is $i$ and $\qb' = \x_{i}$. By definition, the target vector $\vb$ is the token following $\qb'$ in $\X$, i.e., $\vb = \x_{i + 1}$. Let $\Ic = [L'] - \{i, L' - 1\}$. Let $\ab = \sft{\X_k \W\qb} \in \R^{L'}$ be the softmax probabilities where $a_i = a_{L' - 1} = s_q$
\begin{align}
    f(\X) = \sum_{j \in \Ic}a_j\z_j + s_q(\z_i + \z_{L' - 1})
\end{align}
where $\z_j = 2 \sum_{i = -W}^W F_i \x_{j - i}$. We define $R$ as a universal constant and $\Xi = RL\eps_0(1 + N\eps_0) $ such that $\tone{\F} \leq 1 + \Xi $ from Lemma~\ref{lemma filter nar}. Then we get $\tn{\z_j} \leq 2\tone{\F} \leq 2(1 + \Xi)$ for all $j \in [L']$. Note that $a_j /s_q \leq s_v / s_q = \Gamma = \frac{\order{\eps_0 } }{1 - \order{N \eps_0}}$ for all $j \in \Ic$ and that $2s_q + \sum_{j \in \Ic}a_j = 1$. As a result, there exists some constant $R_0 > 0$ such that
\begin{align}
    \frac{1}{2} \geq s_q \geq \frac{1}{2 + (L' - 2)\Gamma} = \frac{1 - \order{N \eps_0}}{2 + (L' - 2N - 2)\order{\eps_0}} \implies |2s_q - 1| \leq R_0L' \eps_0
\end{align}
and \begin{align}
    \sum_{j \in \Ic} a_j = 1 - 2s_q \leq R_0L'\eps_0
\end{align}
Moreover, due to right-clipped convolution, we have $\tn{\z_{L' - 1}} = 2\tone{\sum_{i = 0}^W F_i} \leq 2\Xi$. Next, according to the value retrieval at $i$-th position, we have
\begin{align}
    \tn{\z_i - 2\vb} \leq |2F_{-1} - 2|\tn{\vb} + 2|\sum_{j \neq -1}F_j| \leq 2\Xi
\end{align}
Utilizing these findings above, we get
\begin{align}
    \tn{f(\X) - \vb} 
    &\leq \tn{\sum_{j \in \Ic} a_j \z_j} + \tn{s_q(\z_i + \z_{L' - 1}) - 2s_q \vb} + |2s_q - 1| \tn{\vb}
    \\ &\leq |\sum_{j \in \Ic}a_j| \max_{j}\tn{\z_j} + s_q (\tn{\z_i - 2\vb} + \tn{\z_{L ' -1}}) + |2s_q - 1|
    \\ & \leq 2R_0L'\eps_0(1 + \Xi) + 2\Xi + R_0L'\eps_0
    \\ &\leq 3R_0L'\eps_0 + 2\Xi + 2R_0\Xi L'\eps_0
    \\ & \leq 3\eps_0(R_0L' + 2RL(1 + N\eps_0)(1 + R_0L'\eps_0))
\end{align}
Let $c_0,c_1$ be absolute constants to be determined. Assuming $N\eps_0\leq \order{1}$ (i.e.~bounded by constant), we have that 
\[ 
\tn{f(\X)-\vb}\leq c_0\eps_0(L'+L+LL'\eps_0)
\] 
where $c_0 \geq 3 \max\{R_0, 2R(1 + N\eps_0), 2R_0R(1 + N\eps_0)\}$. Assuming the stronger bound $L\eps_0\leq \order{1}$ and $c_1 \geq c_0(1 + L / L' + L\eps_0) $, we have that 
\[ 
\tn{f(\X)-\vb}\leq c_1\eps_0L'
\] 
This concludes the advertised results.
\end{proof}
\subsection{Proof of Theorem \ref{sel copy thm}}\label{sel copy proof}

Below we state a generalization of Theorem \ref{sel copy thm} which distinguishes two scenarios: Short convolution with PE and Long convolutions with no PE.
\begin{theorem}[Selective Copy]\label{sel copy thm} Consider the setting of Def.~\ref{sel copy def}. There is a 1-layer CAT $f$ using exponential-decay query-convolution $\Fb_q$ as follows:
\begin{itemize} 
\item Suppose $\Fb_q$ is infinitely long (namely parameterized as an SSM with state matrix $\A=\rho$ for some decay parameter $\rho<1$). Then, using $d=|\Sc|+3$ dimensional token embeddings, $f$ solves unique selective copying.
\item Suppose $\Fb_q\in\R^N$ and input sequences contain at most $N$ signal tokens. Using $d=|\Sc|+4$ and 1-D positional encodings, $f$ solves unique selective copying.
\end{itemize}
\end{theorem}

\begin{proof} Let $T$ be the maximum context length the model encounters. Specifically, $T=L+N+1$ where $L$ is the maximum length of the input sequence $X$ that precedes the special token $\bot$ and $N$ is the maximum number of signal tokens in $X$. Recall that the cardinality of the signal vocabulary $\Sc$ is allowed to be larger than $N$ and we resume generation until outputting all signal tokens. Let $Z=[X~\bot~z_{L+2}~\dots~z_{t}]$ denote the current input sequence where $[X~\bot]$ is the initial input that kickstarts decoding. Denote boldface $\Z,\X$ to be the embedded sequences of $Z,X$. We use query convolution thus the CAT model is given by $f(\Z)=\text{nearest\_token\_embedding}(\Z^\top\sft{\Z\W\z^*_t})$ where $\Z^*=\Fb_q\ast\Z$ is the convolved sequence and $\z^*_t$ is the last token of the convolved query sequence for $L+1\leq t\leq T$. We set convolution to be $F_{q,i}=\rho^{i}$ for $0\leq i< W$ for a suitable $\rho\leq 1$ to be determined where $W$ is the window size of the convolution. This choice aggregates the current token and the $W-1$ most recent tokens and allows for all all-ones filter as a special case. For the first statement of the theorem $W=\infty$ whereas for the second statement $W=N$. 


\noindent \textbf{The choice of token embeddings.} We construct the token embeddings as follows:
\begin{itemize}
\item Token embedding of the $i$th token has the form $\x_i=[\x'_i, s_i, p_i]$. Here 
\begin{itemize}
\item \textbf{Base embedding.} $\x'_i$ is the \emph{base embedding} vector associated to the discrete token value $x_i$. We choose these $\x'_i$ embeddings to have unit Euclidean norm. 
\item \textbf{Signal indicator.} $s_i\in\R$ is an indicator of whether the token is a signal token or not. We set $s_i=1$ for signal tokens and the $\bot$ token and $s_i=0$ for noise tokens. 
\item \textbf{Position encoding.} $p_i\in\R$ is the positional encoding of the $i$'th token. We simply set $p_i=i/T$ where $T=L+N+1$. $p_i$ is only required for short convolution i.e.~when $W=N$.
\end{itemize}
\item The base embeddings of noise tokens $\Nc$ are orthogonal to that of signal tokens and $\bot$ token.
\item The base embeddings of signal tokens $\Sc$ and $\bot$ are also orthogonal to each other.
\end{itemize} 
Let $D_{\text{noisy}}$ be the dimension of the subspace spanned by the base embeddings of noise tokens. We can choose $D_{\text{noisy}}=1$ by setting all base embeddings for the noise tokens to be identical. The signal tokens and $\bot$ token occupies an orthogonal $|\Sc|+1$ dimensional subspace. Together, this recovers the embedding dimensions advertised in the theorem statement, namely
\begin{itemize}
\item \textbf{With positional encoding and $W=N$:} We need an embedding dimension of $d=|\Sc|+D_{\text{noisy}}+3\geq |\Sc|+4$ where two additional dimension is due to $s_i$ and $p_i$.
\item \textbf{Without positional encoding and $W=\infty$:} We need an embedding dimension of $d=|\Sc|+D_{\text{noisy}}+2\geq |\Sc|+3$ where the additional dimension is due to $s_i$.
\end{itemize}

\noindent$\bullet$ \textbf{Construction of the CAT model.} We construct a one layer CAT model with the following criteria in the order of priority:
\begin{enumerate}
\item The model should always select signal tokens.
\item The model should select a signal token not previously selected.
\item The model should select the farthest signal token from the current/last token (i.e.~generates signal tokens that are closer to the start of the sequence).
\end{enumerate}
To satisfy the three criteria above, we pick the attention weights $\W$ as follows when $W=N$:
\begin{align}
\W=\begin{bmatrix} -\alpha\Iden_{N+1+D_{\text{noisy}}} & 0 &0\\0 & \beta &0\\0 &0&-\theta\end{bmatrix}.\label{w att choice}
\end{align}
The choice for $W=\infty$ is same except that we do not have the positional encoding coefficient $\theta$. Recall that we also choose the convolutional filter as $F_{q,i}=\rho^{i}$ for $0\leq i< W$ for $\rho<1$. Specifically, we choose $\rho=2^{-1/T}$ so that $\rho^T=1/2$. This choice guarantees that $\rho^i-\rho^{i+1}\geq c/T$ for all $0\leq i<T$ for some absolute constant $c>0$.

We will accomplish the proof inductively. Suppose that $X$ contains $N'$ unique signal tokens and that until time $t$ for some $L+1\leq t\leq L+N'+1$, the model outputs the correct sequence of $t'=t-L-1$ unique signal tokens. We will prove that it will accurately output the next signal token in line with suitable choice of $\alpha,\beta,\theta$. To this end, we state the following lemma regarding the output of the query-convolution $\z^*_t$.

Note that $\z^*_t=\sum_{i=1}^t \rho^i \z_{t-i}$. Recall that $z_{L+1}$ to $z_{t}$ are unique correctly-ordered signal tokens where we set $z_0=\bot$. Denote the rest of the $N'-t'$ signal tokens with correct order by $q_1$ to $q_{N'-t'}$. Here $q_1$ is the left most such token in $X$ and the token we wish to output next. We can write $\z^*_t$ in terms of signal tokens and noise tokens as follows:
\begin{align}
\z^*_t=\sum_{i=1}^{t'+1} b_i \z_{t-i}+\nb+\sum_{j=1}^{N'-t'}a_j\qb_j,\label{aibieq}
\end{align}
where we set $b_i:=\rho^i$. Here the first term $\sum_{i=1}^{t'+1} b_i \z_{t-i}$ is due to last $t'+1$ tokens (including $\bot$) that are already generated. The $\nb$ term denotes the aggregated contribution of the noise tokens to the convolution. $\sum_{j=1}^{N'-t'}a_j\qb_j$ is the contributions of the signal tokens that are yet to be generated. Crucially note that, 
\begin{itemize}
\item If $W=\infty$, $a_i$ is strictly increasing because convolution coefficients $F_{q,j}$ are strictly decreasing (with a gap lower bounded by $c/T$).
\item Whether $W=N$ or $W=\infty$, $b_i=\rho^i$ is strictly decreasing and $b_{t'+1}\geq a_{N'-t'}+c/T$. That is, the contribution of any token already generated is larger than any token that is yet to be generated.
\end{itemize}

Let us write $\z^*_t=[\z'^*_t~s~p]$. Note that $s\geq 1/2$ because $\bot$ token is involved in the convolution and $\rho^T=1/2$. Similarly, if we employ PE, we have that $p\geq (L+1)/2T\geq 1/4$ for the same reason. Given a token $\x_i=[\x'_i~s_i~p_i]$, through \eqref{w att choice}, we have that
\begin{align}
\text{score}_i=\x_i^\top\W\z^*_t=-\alpha \li\z'^*_t, \x'_i\ri+\beta s s_i-\theta p p_i.\label{score eq}
\end{align}
We now proceed with the proof which relies on choosing $\alpha,\beta,\theta>0$ in a suitable fashion. Specifically, we will choose their relative ratios $\beta/\alpha,\alpha/\theta$ suitably to ensure the desired token $q_1$ receives the highest score. After ensuring this, we can suitably scale up $\alpha,\beta,\theta$ in a proportional fashion, which will also scale up the scores of each token. Thanks to softmax attention, this will ensure that the model precisely retrieves the token with the highest score.

\noindent\textbf{Scenario 1: $W=\infty$.} In this scenario, we don't use PE, thus, effectively $\theta=0$. We need to satisfy aforementioned criteria: First, we want the highest score to be a signal token. We will guarantee this by observing $s_i=0$ for noise tokens, $s>0$ and by setting $\beta/\alpha\gg 1$. Second, we want the highest score to be $q_1$, the left most signal token that has not been output yet. Now, since $W=\infty$, $q_1$ receives the lowest coefficient of $a_1$ in \eqref{aibieq}. Using orthogonality and unit Euclidean norm, this implies that $\li\z'^*_t, \qb'_1\ri=a_1$. In contrast, any other signal token has a larger inner product by at least $c/T$. Choosing $\alpha=1$ (and then suitably scaling it up together with $\beta$), this implies that, $q_1$ is indeed the token with the highest score that will be generated next.

\noindent\textbf{Scenario 2: $W=N$ and we employ PE.} We again follow the score decomposition \eqref{score eq}. Observe that $\li\z'^*_t, \x'_i\ri, ss_i,pp_i$ are all bounded by 1 in absolute value. Thus, by controlling the relative ratios of the scalar knobs $\beta>\alpha>\theta=1$, we can enforce the three criteria listed above. Recall that $q_1$ denotes the next signal token we wish to output next. We will prove that $q_1$ achieves the strictly highest score amoung the tokens of $Z$. To proceed, set $\beta/\alpha\gg 1$ and $\alpha/\theta\gg1$. 
\begin{itemize}
\item Since $\beta$ dominates $\alpha$ and $\theta$, following the same argument in Scenario 1, noise tokens will have strictly lower scores than signal tokens, thus cannot be generated next.
\item Following \eqref{score eq}, the signal tokens have a score contribution of $-\alpha\cdot b_i$ or $-\alpha\cdot a_j$ from the inner product term $\li\z'^*_t, \x'_i\ri$. Here $b_i$ denoted the coefficient of a generated signal token whereas $a_j$ denoted the coefficient of a missing signal token. Next recall from \eqref{aibieq} that $b_i\geq c/T>0$ and $b_i\geq a_j+c/T$ thanks to the $\Fb_q$ choice. Since $\alpha$ dominates $\theta$, this implies that the generated signal tokens have strictly less score than the missing signal tokens.
\item Finally, we wish to show that $q_1$ has the highest score among missing signal tokens. First, recall from \eqref{aibieq} that $a_1$ is the smallest coefficient among the missing signal tokens. As a result, it achieves the largest inner product score $-\alpha\cdot \li\z'^*_t, \x'_i\ri$. To complete the proof, we use positional encoding to break any score ties. Since $q_1$ is the left most missing signal token, any other missing signal token will achieve a strictly worse position encoding score $-\theta p p_i$ as $p\geq 1/4$, $\theta=1$, and $p_i=i/T$ is strictly increasing. This guarantees that $q_1$ achieves the strictly highest score as desired.
\end{itemize}
To summarize, by choosing suitable $\beta\gg \alpha\gg\theta=1$ and proportionally scaling up $\alpha,\beta,\theta$ sufficiently, we conclude with the proof.
\end{proof}


\section{Proofs for Section \ref{sec tradeoff} -- Convolution-Attention Tradeoff}

\subsection{Proof of Proposition \ref{prop long conv1}}\label{single q proof}
The key to the proof is establishing the detection threshold of the correct block ID $\beta$ in \eqref{hcateq} i.e.~we wish to guarantee $b=\beta$. Once correct block is retrieved, the rest of the argument is identical to AR over dense attention as we retrieve the correct blocks. Observe that, we have $\barL-1$ blocks in total (not counting the local/final block). Note that $\z_i\sim\Nn(0,\sigma^2B\Iden_d/d)$ for $i\neq \beta$ and $\z_\beta\sim\Nn(\x_L,\sigma^2(B-1)\Iden_d/d)$.

Set $g_i=\z_i^\top\x_L\cdot\sqrt{d/B}$ for $i<\bar{L}$ and $g_\beta$. Observe that $g_i$'s and $g_\beta$ are independent random variables. Additionally, $g_{i\neq\beta}\sim \Nn(0,\sigma^2)$, $g_\beta\sim\Nn(\sqrt{d/B},(1-1/B)\sigma^2)$. Let $g_{\max}=\max_{i\neq \beta} g_i$. We have the following gaussian concentration inequalities
\begin{align}
&\Pro(g_{\max}\geq \sigma(\sqrt{2\log L'}+t))\leq e^{-t^2/2}\\
&\Pro(g_\beta\leq \sqrt{d/B}-\sigma t)\leq e^{-t^2/2}.
\end{align}
Combining these three, we find that, with probability $1-2e^{-t^2/2}$, whenever $\sqrt{d/B}\geq \sigma(\sqrt{2\log L'}+2t)$, we have that
\[ 
g_\beta>g_{\max}=\max_{i\neq b}g_i.
\]
This condition is implied by $d\geq \sigma^2B(\sqrt{2\log \barL}+2t)^2$. Applying change of variable on $t$, we conclude with the result. 

\noindent\textbf{Retrieving the value token.} Once the correct block is identified, (query, value) pair is retrieved by applying full softmax attention with $\W=c\Iden$ with $c\rightarrow\infty$ within the selected two blocks. Recall that local attention retrieves the query and the choice of convolutional filter will return the value ahead of the query. 
To guarantee this, we only need to prove that $\x_L$ also has the largest correlation to itself within the two selected blocks we apply local attention. To this aim, we similarly use the fact that correlations between $\x_L$ and the other tokens in the selected blocks are IID $\Nn(0,\sigma^2/d)$ variables. There are at most $2B-2$ such other tokens. Consequently, the maximum local correlation $g^{\text{loc}}_{\max}$ obeys $\Pro(g^{\text{loc}}_{\max}\geq \sigma(\sqrt{2\log(2B)}+t)/\sqrt{d})\leq e^{-t^2/2}$. We wish to guarantee that $g^{\text{loc}}_{\max}<1$. This holds with $1-e^{-t^2/2}$ probability whenever $d\geq \sigma^2(\sqrt{2\log(2B)}+t)^2$. This latter condition is implied by the original condition because $\sqrt{B}(\sqrt{2\log \barL}+2t)\geq \sqrt{2B\log 2}+2t\geq \sqrt{2\log(2B)}+t$. Union bounding, we end up with a success probability of at least $1-3e^{-t^2/4}$.

Next, we wish to show the converse result. We recall that as the expectation of supremum of $K$ IID $\Nn(0,1)$ become $(1\pm o(1))\sqrt{2\log K}$ as $K$ grows to infinity. Thus, for sufficiently large $\barL\geq C_\eps$, we have that $\E[g_{\max}]\geq \sqrt{(2-\eps)\log \barL}$. Consequently, we can write the reversed inequalities
\begin{align}
&\Pro(g_{\max}\leq \sigma(\sqrt{(2-\eps)\log \barL}-t))\leq e^{-t^2/2}\\
&\Pro(g_\beta\geq \sqrt{d/B}+\sigma t)\leq e^{-t^2/2}.
\end{align}
Combining these, we conclude with the advertised reverse inequality. As a result, with the same probability, we fail to identify the block containing the target query/value pair.

The next subsection proves the uniform AR guarantees via an application of Slepian's lemma.

\subsection{Proof of Uniform Associative Recall via Slepian's Lemma (Proposition \ref{prop long conv1} continued)}\label{unif q proof}

Slepian's Lemma \cite{slepian1962one} is an important gaussian comparison inequality. A specific variation is the following result that holds for a random gaussian matrix. We first introduce the Gaussian width definition that is important for assessing the complexity of a geometric set in a high-dimensional space.
\begin{definition}[Gaussian width] Let $S\in\R^d$ and $\g\sim\Nn(0,\Iden_d)$. Gaussian width of $S$ is defined as $\omega(S)=\E[\sup_{\x\in S}\x^\top \g]$
\end{definition}
\begin{proposition}[Slepian's Lemma] Let $\X\in\R^{n\times d}$ be a matrix with IID $\Nn(0,1)$ entries. Let $\g\sim\Nn(0,\Iden_n)$ and $\h\sim\Nn(0,\Iden_d)$. Given sets $A\in\R^n,B\in\R^d$, we have that
\[ 
\E[sup_{\ab\in A,\bb\in B}\ab^\top \X^\top \bb]\leq \E[sup_{\ab\in A,\bb\in B}\ab^\top \g\tn{\bb}+\bb^\top\h\tn{\ab}].
\]
\end{proposition}
We have the following application of Slepian's lemma.
\begin{lemma}\label{lemma slep version} Let $\X\in\R^{L\times d}$ be a matrix with IID $\Nn(0,1)$ entries. Let $g\sim\Nn(0,\Iden_d)$ be an independent vector. Fix a subset of unit sphere $S\in\R^d$. With probability $1-e^{-t^2/2}$, we have that
\[ 
\sup_{\bt\in S}\tin{\X\bt}\leq \sqrt{2}(\sqrt{\log L}+\omega(S)+t).
\]
\end{lemma}
\begin{proof} Define the augmented matrix $\X'=\begin{bmatrix}\X\\-\g\end{bmatrix}$. Define the set $A\in\R^{L+1}$ such that $\ab\in A$ has the following form. $\ab$ has two nonzero entries both of which are equal to $1$. Additionally, last entry is nonzero i.e.~$\ab_{L+1}=1$. Using $\tn{\ab}=\sqrt{2}$ and $\tn{\bt}=1$, we now apply Slepian's lemma as follows
\begin{align}
\E[sup_{\ab\in A,\bt\in S}\ab^\top \X^\top \bt]&\leq \E[sup_{\ab\in A,\bb\in B}\ab^\top \g\tn{\bt}+\bt^\top\h\tn{\ab}]\\
&\leq \E[\tin{\g}+\sqrt{2}\sup_{\bt\in S}\bt^\top h]\\
&\leq \sqrt{2}(\sqrt{\log L}+\omega(S)).
\end{align}
To proceed, observe that $\ab^\top \X^\top \bt$ is a $\sqrt{2}$-Lipschitz function of $\X$. This implies the statement of the lemma.
\end{proof}

\begin{proposition}\label{prop slepian} Consider the setting of Proposition \ref{prop long conv1}. Suppose we wish to solve AR for the worst query drawn from a set $S$ which is subset of the unit sphere. If $d\geq 2\sigma^2 B(\sqrt{\log\barL}+\omega(S)+t)^2$, \eqref{hcateq} solves AR with probability at least $1-2e^{-t^2/2}$ for all $\x_L\in S$.
\end{proposition}
\begin{proof} The proof follows the steps of Section \ref{single q proof} with the following distinction. Note that, to determine the correct block, we now need to do a worst case analysis. Namely, let $\z'_\beta=\x_L-\z_\beta$, $\z'_i=\z_i$ for $i\neq \beta$, and set $\Z'=\begin{bmatrix}\z'_1\\\vdots\\\z_{\barL-1}\end{bmatrix}$. Also let $\Z=[\z_1~\dots~\z_{\beta-1}~\z_{\beta+1}~\dots~\z_{\barL-1}]$. The accurate detection of the block $\beta$ coincides with the following event
\[ 
\inf_{\x_L\in S}\tin{\Z\x_L}-\z_\beta^\top\x_L>0.
\]
Using $\z^\top\x_L=1-\z'^\top \x_L$ and defining the set $A$ to be the set of all vectors with exactly two 1s with one of the 1 appearing at position $\beta$, the above event can alternatively be written as 
\[ 
\sup_{\ab\in A,\x_L\in S}\ab^\top\Z'\x_L<1.
\]
Now applying Lemma \ref{lemma slep version} on the left hand side, we find that, with probability $1-e^{-t^2/4}$,
\[ 
\sup_{\ab\in A,\x_L\in S}\ab^\top\Z'\x_L\leq \sqrt{\frac{2B}{d}}\sigma(\sqrt{\log L}+\omega(S)+t)
\]
Consequently, whenever $d>2\sigma^2 B(\log L+\omega(S)+t)^2$, we conclude with the result. Note that, when $S$ is an $r$-dimensional subspace, we plug in the well-known bound $\omega(S)\leq \sqrt{r}$. Finally, we need to union bound this event with the event that the query token can be identified through local attention by letting $\W_k=\W_q=\sqrt{c}\Iden$ and $c\rightarrow\infty$. To do so, we apply Lemma \ref{lemma slep version} over the $2B-2$ non-query tokens. Denoting these tokens by $\X^{loc}\in\R^{d\times (2B-2)}$, we have that $\Pro(\X^{loc}\qb\geq \sqrt{2\sigma^2/d} \cdot(\sqrt{\log(2B)}+\omega(S)+t))\leq e^{-t^2/2}$. Consequently, $\X^{loc}\qb<\qb^\top\qb=1$ as soon as the same condition $d\geq 2\sigma^2 B(\sqrt{\log\barL}+\omega(S)+t)^2$ holds. This introduces an additional $e^{-t^2/4}$ probability of failure.
\end{proof}

\subsection{Proof of Proposition \ref{prop long conv2}}
We essentially follow the proof of Proposition \eqref{prop long conv1}. The only differences are that, the variance calculations, comparison of block correlations, and signal-to-noise ratio bounds will all slightly change due to exponential smoothing impacting the full context window. To proceed, let us observe the following scenarios for a block ID $1\leq i< \barL$:
\begin{itemize}
\item \textbf{Scenario 1: $i<\beta$.} In this scenario, $\z_i$ is exponentially-smoothed sum of IID vectors with $\Nn(0,1)$ entries. Recalling $\rho=e^{-1/B}$, the variance $\sigma_z^2$ of entries of $\z_i$ is upper bounded by
\begin{align}
\sigma_z^2=\sum_{i=0}^\infty \rho^{2i}=\frac{\sigma^2}{1-\rho^2}\leq 1.2B.\label{var bound}
\end{align}
Here, we used the fact that for $B=1$, the bound holds and for $B\geq 2$, we have that $\rho^2=e^{-2/B}\leq 1-\frac{1}{B}$. The latter implies $1-\rho^2\geq 1/B$ and $1/(1-\rho^2)\leq B$. 

The above bound on $\sigma_z^2$ implies that, setting $g_i=\z_i^\top\x_L\cdot\sqrt{d/B}$, we have that $g_i\sim\Nn(0,\sigma_i^2)$ with $\sigma_i^2\leq 1.2\sigma^2$.

\item \textbf{Scenario 2: $i=\beta$.} In this scenario, the variance upper bound $\sigma_i^2$ above is still applicable. The key is to estimate and lower bound the mean component similar to the proof of Proposition \eqref{prop long conv1}. Let the query token appear in the $k$th position of block $\beta$ for $k\in[B]$. Define $p=(k-1)/B$. Observe that
\[ 
\E[g_\beta]=\E[\z_\beta^\top\x_L\cdot\sqrt{d/B}]=e^{-p}\sqrt{d/B}.
\]

\item \textbf{Scenario 3: $i>\beta$.} This is essentially same as Scenario 2, because, thanks to the exponential smoothing, the signal token from block $\beta$ will propagate to future $\z_i$'s. The coefficient of the propagation satisfies 
\[ 
\E[g_i]=\E[\z_i^\top\x_L\cdot\sqrt{d/B}]=e^{-(p+i-\beta)}\sqrt{d/B}.
\]
\end{itemize}
Now that we have gathered these three scenarios, we can define $g_{\max}=\max_{i\neq \beta}g_i-\E[g_i])$ similar to above. $g_{\max}$ is a supremum of independent Gaussians of bounded variance controlled by \eqref{var bound}. Through this, we have that
\begin{align}
&\Pro(g_{\max}\geq 1.6\sigma(\sqrt{\log\barL}+t))\leq e^{-t^2}\\
&\Pro(g_\beta-\E[g_\beta]\leq 1.6\sigma t)\leq e^{-t^2}.
\end{align}
Secondly, for $i\neq\beta$, using $p\leq 1$, we have that
\[ 
\E[g_\beta]-\E[g_i]\geq (e^{-p}-e^{-(p+i-\beta)})\sqrt{d/B}\geq (e^{-1}-e^{-2})\sqrt{d/B}\geq 0.23\sqrt{d/B}.
\]
Consequently, we require $0.23\sqrt{d/B}>1.6\sigma(\sqrt{\log\barL}+2t)$. Using $\tau=t/2$, this is guaranteed by $d\geq 50B\sigma^2(\sqrt{\log\barL}+\tau)^2$ with probability at least $1-2e^{-\tau^2/4}$. Once the correct block is identified, (query, value) pair is retrieved by applying dense softmax attention with $\W=c\Iden$ with $c\rightarrow\infty$ over the selected blocks thanks to the choice of convolutional filter. This argument is identical to ``Retrieving the value token'' proof in Section \ref{single q proof} and introduces an additional $e^{-\tau^2/2}$ probability of failure in the union bound. 

\end{document}